\documentclass[12pt,a4paper,UTF8]{article}
\usepackage{graphicx} 
\usepackage{listings} 
\usepackage{color} 
\usepackage{booktabs} 
\usepackage{pdflscape} 
\usepackage{hyperref} 
\usepackage{xcolor}
\usepackage{amssymb}
\usepackage{amsmath} 
\usepackage{mathrsfs}
\usepackage{geometry}
\usepackage{subcaption}
\usepackage{float}

\usepackage[numbers, sort&compress, square, comma]{natbib}
\usepackage{amsthm}

\usepackage{algorithm}
\usepackage{algpseudocode}

\usepackage{fontspec}
\usepackage{fancyhdr}
\definecolor{lightblue}{RGB}{173, 216, 230} 

\newif\ifshowrevision
\showrevisionfalse   

\newcommand{\TG}[1]{{\color{black}#1}}

\ifshowrevision
  \newcommand{\revision}[1]{{\color{blue}#1}}
\else
  \newcommand{\revision}[1]{#1}  
\fi

\newtheorem{definition}{Definition}
\newtheorem{property}{Property}
\newtheorem{theorem}{Theorem}
\newtheorem{corollary}{Corollary}

\author{
 \normalsize{
Yixin Wang$^{1,2,}$\footnotemark[2]\ ,
Ting Gao$^{1,2,3}$\footnotemark[1]\ ,
Jinqiao Duan$^{4,5,}$\footnotemark[3]
}\\[10pt]
\footnotesize{$^1$ School of Mathematics and Statistics, Huazhong University of Science and Technology, Wuhan, China} \\
\footnotesize{$^2$ Center for Mathematical Science, Huazhong University of Science and Technology, Wuhan, China} \\
\footnotesize{$^3$ Steklov-Wuhan Institute for Mathematical Exploration, Huazhong University of Science and Technology, China} \\
\footnotesize{$^4$ Department of Mathematics and Department of Physics, Great Bay University, Dongguan, China} \\
\footnotesize{$^5$ Guangdong Provincial Key Laboratory of Mathematical and Neural Dynamical Systems, Dongguan, China. }
}

\begin{document}

    \title{Beyond Distance: Quantifying Point Cloud Dynamics with Persistent Homology and Dynamic Optimal Transport}
    \date{}
    \maketitle
    \begin{abstract}
We introduce a framework for analyzing topological tipping in time evolutionary point clouds by extending the recently proposed Topological Optimal Transport (TpOT) distance. While TpOT unifies geometric, homological and higher‐order relations into one metric, its global scalar distance \revision{can obscure transient, localized structural reorganizations during dynamic phase transitions}. To overcome this limitation, we \revision{present a hierarchical dynamic evaluation framework driven by a novel topological and hypergraph reconstruction strategy}. \revision{Instead of directly interpolating abstract network parameters, our method interpolates the underlying spatial geometry and rigorously re-computes the valid topological structures, ensuring physical fidelity}. Along this geodesic, \revision{we introduce a set of multi-scale indicators: macroscopic metrics (Topological Distortion and Persistence Entropy) to capture global shifts, and a novel mesoscopic dual-perspective Hypergraph Entropy (node-perspective and edge-perspective) to detect highly sensitive, asynchronous local rewirings}. We further propagate \revision{the cycle-level} entropy change onto individual vertices \revision{to form a point-level topological field}. \revision{Extensive evaluations on physical dynamical systems (Rayleigh-Van der Pol limit cycles, Double-Well cluster fusion), high-dimensional biological aggregation (D'Orsogna model), and longitudinal stroke fMRI data} demonstrate the utility of combining transport‐based alignment with multi‐scale entropy diagnostics for dynamic topological analysis.

    \par\textbf{Keywords: }Dynamic Topological Optimal Transport, Multiscale Tipping, Persistence entropy, Hypergraph entropy, Medical imaging
    
\end{abstract}
    \section{Introduction}
    \hspace{1.5em} Complex systems (e.g., climate systems, biological networks, financial markets) often exhibit nonlinear phase transitions and critical \TG{phenomena} \cite{Feng2023EarlyWI}. Detecting their tipping points and structural changes is crucial for predicting systemic collapse and designing early-warning mechanisms \cite{Zhang2023EARLYWP}. Traditional statistical metrics, \TG{such as variance and autocorrelation coefficients, are employed to monitor} critical transitions and their driving networks \cite{liu2013dynamical}; \TG{however, they overlook} potential topological structural changes in dynamical systems. \TG{In contrast}, dynamical models like bifurcation theory require predefined \TG{differential} equations and perform poorly on real-world non-stationary systems with unknown couplings \cite{kuehn2011mathematical}. In \cite{Zhang2024ActionFA}, the authors introduce a novel Schrödinger bridge approach for early warning signals (EWS) in probability measures that align with the entropy production rate. \TG{Nevertheless}, aligning disparate distributions in a manner that is both mathematically rigorous and computationally tractable remains a fundamental challenge in interdisciplinary science.

    Originating from Monge's 1781 optimal transport problem, optimal transport (OT) \TG{theory} seeks \TG{to find couplings between} probability distributions as efficiently as possible with respect to a given cost function \cite{villani2008optimal}. A well-known distance measure, the Wasserstein distance, evaluates the geometric separation between supports, where the cost function \TG{is} induced by the distance function in a metric space. Recently, extensions of the optimal transport problem, such as the Gromov–Wasserstein (GW) distance and its variant, the Fused Gromov–Wasserstein (FGW) distance, have garnered increasing attention \TG{due to their applicability} in settings where probability distributions are defined on different metric spaces \cite{vayer2020fused, memoli2007use, memoli2011gromov}. Motivated by the goal of representing relations in complex systems through higher-order networks, recent work has applied a Gromov–Wasserstein variant known as co-optimal transport \cite{titouan2020co} to hypergraph modeling, demonstrating its effectiveness in both theoretical frameworks and practical applications \cite{chowdhury2024hypergraph}.

 Topological Data Analysis (TDA) is a mathematical framework rooted in algebraic topology that extracts multiscale topological features from high-dimensional complex data \cite{wasserman2018topological}. Giusti and Lee developed a computable feature map for paths of persistence diagrams \cite{giusti2023signatures}. Li et al. \TG{introduced} a flexible and probabilistic framework for tracking topological features in time-varying scalar fields by employing merge trees and partial optimal transport \cite{li2025flexible}. Numerous researchers have developed innovative methodologies for topological feature tracking and event detection by leveraging diverse tools from TDA, including persistence diagrams, merge trees, Reeb graphs, and Morse–Smale complexes. For example, Tanweer et al. introduced a TDA-based approach that uses superlevel persistence to mathematically quantify P-type bifurcations in stochastic systems through a "homological bifurcation plot," which \TG{illustrates} the changing ranks of 0th and 1st homology groups via Betti vectors \cite{tanweer2024topological}. Shamir et al. proposed a progressive isosurface algorithm that predicts the contour at time step $t+1$ based on the contour at time step $t$\cite{bajaj2002progressive}. Doraiswamy et al. described a framework for the exploration of cloud systems at \TG{multiple spatial and temporal} scales \TG{using} infrared(IR) brightness temperature images, which automatically extracts cloud clusters as contours for a given temperature threshold\cite{doraiswamy2013exploration}. A comprehensive review of methodologies for topological feature tracking and structural change detection is presented in \cite{yan2021scalar}.

 In recent years, \TG{there has been growing interest in} extending classical graph entropy\revision{\cite{Trucco1956ANO,Rashevsky1955LifeIT,dehmer2008novel,Dehmer2008InformationPI}} to hypergraphs, as hypergraphs are capable of capturing higher-order interactions that conventional pairwise networks cannot. The \TG{concept} of hypergraph entropy was first introduced by Simonyi (1996) \cite{simonyi1996entropy} within an information-theoretic framework. \TG{Subsequent studies have proposed alternative formulations, including} entropy vectors derived from partial hypergraphs \cite{bloch2019new} and tensor-based entropy for uniform hypergraphs \cite{chen2020tensor}. \TG{Additionally}, entropy-maximization models \TG{have been employed} to generate random hypergraphs, serving as useful null models for real-world systems \cite{saracco2025entropy}. These developments underscore the versatility of hypergraph entropy as a tool for quantifying uncertainty and complexity in higher-order networks, motivating its application to the analysis of hypergraph \TG{structures} derived from persistent homology.

In this work, we propose \TG{a set of dynamic distortion and} entropy indicators that integrate optimal transport, topological data analysis, and \TG{information theory for} hypergraphs. Our approach builds upon the recently introduced Topological Optimal Transport (TpOT) framework \cite{zhang2025topological}, which aligns point clouds while jointly optimizing geometric correspondence and topological fidelity through a principled coupling of their persistent homology classes. By \TG{incorporating} this trade-off between geometric and topological preservation, our metrics enable robust detection of multiscale structural shifts. \revision{To systematically capture these multiscale dynamics, our methodology proceeds in four fundamental stages (Figure \ref{fig:geodesic_interpolation}): computing the initial TpOT spatial coupling, performing geometric interpolation along the underlying geodesic, rigorously reconstructing valid topological and hypergraph structures to ensure physical fidelity, and extracting multiscale early warning indicators. Driven by this pipeline, our main contributions are listed as follows. First, we propose a dynamic hypergraph reconstruction strategy that integrate interpolation of underlying geometric information. Second, we propose dynamic distortion and entropy as early warning indicators for multi-scale tipping detection, notably a novel mesoscopic dual-perspective hypergraph entropy ($\mathrm{HE}_V$ and $\mathrm{HE}_E$) that uniquely captures asynchronous structural decoupling. We further propagate the cyle-level entropy change onto individual vertices to form a point-level topological field that identifies key local transformations. Third, we provide rigorous theoretical guarantees for these metrics, which are precisely validated across physical dynamical systems, biological aggregations, and clinical stroke fMRI data.

The structure of this paper is organized as follows: Section \ref{sec:methodology} describes our proposed topological \& hypergraph reconstruction and dynamic distortions \& entropies framework for multiscale tipping detection. Especially, Sections \ref{ssec:geodesic_losses} to \ref{ssec:point-level-entropy} detail the dynamic reconstruction strategy, the dual-perspective entropy formulations with their theoretical proofs, and the point-level localized mapping, respectively. Section \ref{sec:experiments} presents comprehensive experimental validations, finally we summarize conclusions and future work in Section \ref{sec:conclusion}.}

\section{Methodology}\label{sec:methodology}
\hspace{1.5em} In many applications—from neuroscience and materials science to machine learning—persistent homology (PH) has proven effective for dimension reduction, feature generation, and hypothesis testing. However, standard diagram-based comparisons often disregard the underlying geometry of representative cycles. Recent extensions \TG{address this limitation} by extracting explicit cycles and organizing them into higher-order structures such as PH-hypergraphs, where vertices \TG{correspond to} data points and hyperedges encode cycle memberships \cite{barbensi2022hypergraphs}. It has also been shown that the Topological Optimal Transport (TpOT) framework \TG{not only} provides a powerful distance for comparing complex data but also endows the resulting space with rich geometric structure, enabling interpolation, extrapolation, and barycenter computations among measure–topological networks \cite{zhang2025topological}. Detailed mathematical definitions and properties are provided in the Supplementary Material~\ref{Supplementary material:TpOT}.

Before introducing our extensions, we critically assess the properties of the Topological Optimal Transport (TpOT) framework. This analysis motivates the design choices presented in this section.

TpOT simultaneously aligns point–point affinities via a Gromov–Wasserstein term, persistence diagram coordinates via a Wasserstein term, and point–cycle incidences via a co-optimal transport term, thereby integrating geometric, homological, and higher-order relational information within a single optimization. The space of measure topological networks \(\bigl(\mathcal{P}/\!\!\sim_w,\,d_{\mathrm{TpOT},p}\bigr)\) admits geodesics that are realized by convex combinations of kernels, birth–death embeddings, and incidence functions.  These geodesics enjoy non‐negative Alexandrov curvature. As a pseudo-metric, \(d_{\mathrm{TpOT},p}\) satisfies symmetry and the triangle inequality. The non-negative curvature property of the geodesic further guarantees convexity of squared‐distance functionals, which enables robust interpolation, barycenter computations, and clustering in the network space. However, the TpOT faces two limitations:
    \begin{itemize}
        \item \textbf{Static distortion loss masks evolution details.}
    The TpOT framework returns only \TG{a} scalar distance \TG{that quantifies} the cost of transporting \(P\) to \(P'\) along the optimal geodesic.  However, this endpoint‐only value omits the information about the intermediate distortions
    \(\mathcal{L}_{\mathrm{geom}}^t\), \(\mathcal{L}_{\mathrm{topo}}^t\), and \(\mathcal{L}_{\mathrm{hyper}}^t\)
    for \(t\in(0,1)\). As a result, any transient or evolving geometric or topological phenomena occurring between \(P\) and \(P'\) are \TG{entirely} masked, preventing \TG{the} localization of when—and \TG{to what extent} —significant structural changes occur along the interpolation path.  
        
    \item \textbf{Difficulty of detecting abrupt topological changes.} Constant-speed interpolation \TG{yields} smoothly varying loss curves. Consequently, abrupt topological changes, such as the sudden birth or death of long-lived cycles, \TG{appear} only as mild inflections, making them difficult to detect \TG{within} the aggregated TpOT objective.

    \item \textbf{Sensitivity of hyperparameters.} The choice of the entropic regularization weight \(\varepsilon\) critically \TG{influences} the \TG{trade-off} between fidelity and smoothness. Inappropriate settings may \TG{result in} overly diffuse couplings that obscure salient structure, or in unstable transport plans \TG{that are highly} sensitive to noise.
    \end{itemize}

    To \TG{address} these inherent limitations, we \TG{introduce two key extensions} in this work:

    \begin{itemize}
      \item \textbf{Dynamic distortion curves.}  
        Rather than report only the \TG{final scalar distance} \(d_{\mathrm{TpOT},p}\), we extract, at each interpolation time \(t\in[0,1]\), the three dynamic distortion loss functions
        \(\mathcal{L}_{\mathrm{geom}}^t\), \(\mathcal{L}_{\mathrm{topo}}^t\), and \(\mathcal{L}_{\mathrm{hyper}}^t\).  
        These curves \TG{precisely localize the timing} and magnitude of geometric, topological, and incidence deformations \TG{along the geodesic}.
    
      \item \textbf{Entropy‐based event detectors.} 
      We introduce \emph{persistence entropy} (PE) and \emph{hypergraph entropy} (HE) as complementary diagnostics. PE quantifies the Shannon entropy of the \TG{distribution of} barcode lengths, \TG{capturing} sensitivity to the abrupt appearance or disappearance of significant cycles. HE \TG{measures} the \TG{uniformity of} vertex–hyperedge incidence, highlighting sudden reorganizations in higher-order connectivity.
        
    \end{itemize}
    
    \TG{By integrating} these entropy-based indicators, we are able to continuously monitor and detect anomalous topological events in temporally evolving point clouds. In this section, we provide a detailed \TG{description} of the proposed pipeline for quantifying \TG{multi-scale tipping phenomena with geometric/topological/hypergraph distortions in time-evolving point clouds, which is illustrated in  Figure \ref{fig:geodesic_interpolation}}. 
    
    \TG{The structure of this section is as follow: First, we present the background knowledge of Measure Topology Network and TpOT for optimal coupling in Section \ref{ssec:network_construction} and Section \ref{ssec:tpot_compute}. Then, our main contributions on Hypergraph Reconstruction and early warning indicators with dynamic distortions are presented in Section \ref{ssec:geodesic_losses}, which explained tipping detection in four-steps aligned with  Figure \ref{fig:geodesic_interpolation}. Furthermore, we study various definitions of entropy as early warning indicators in Section \ref{ssec:entropy}. Our contribution also lies in novel definitions of Vertex-Perspective entropy, Hyperedge-Perspective entropy and Symmetric Hypergraph entropy, with theoretical proofs on sensitivity of abrupt topological transitions and topological upper bound. In addition, we also present point-level hypergraph entropy in Section \ref{ssec:point-level-entropy}, as an important evaluation metric in some real data experiment. }

    \begin{figure}[htbp]
        \centering
        \includegraphics[width=\linewidth]{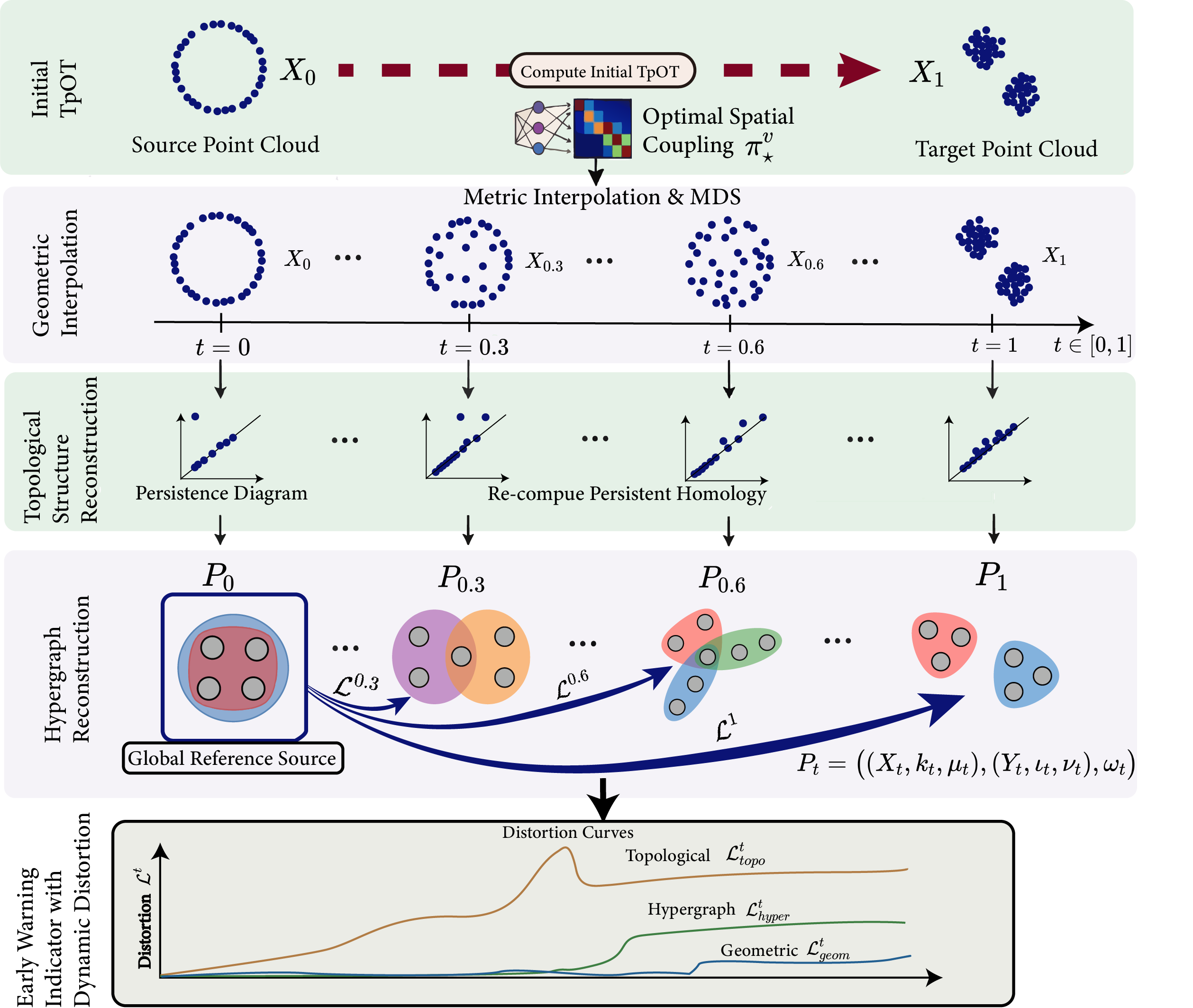}
        \caption{\revision{\textbf{Schematic overview of the hypergraph reconstruction and dynamic distortion framework.}  \textbf{Row 1 (Initial TpOT):} The optimal spatial coupling $\pi^v_\star$ is computed between the source $X_0$ and target $X_1$ point clouds via solving optimal TpOT problem. \textbf{Row 2 (Geometric Interpolation):} Metric interpolation and Multidimensional Scaling (MDS) generate intermediate spatial configurations $X_t$ along the geodesic $t \in [0,1]$. \textbf{Row 3 (Topological Structure Reconstruction):} Persistent homology is computed based on the interpolated geometric point cloud to extract authentic intermediate features. \textbf{Row 4 (Hypergraph Reconstruction:} Measure topological networks $P_t$ are assembled, where nodes represent regions and colored hulls represent topological hyperedges.  \textbf{Row 5 (Dynamic Distortion as Early Warning Indicators)}The dynamic distortions ($\mathcal{L}^t$) are rigorously computed by solving the TpOT distances between the global reference source $P_0$ and each intermediate state $P_t$. }}
        \label{fig:geodesic_interpolation}
    \end{figure}

\subsection{Measure Topological Network}\label{ssec:network_construction}

Given a point cloud \revision{i.e., a finite set of points} \(X=\{x_i\}_{i=1}^N\subset\mathbb{R}^d\), we build its associated \emph{measure topological network}\cite{zhang2025topological}
\[
P \;=\;\bigl((X,k,\mu),\;(Y,\iota,\nu),\;\omega\bigr),
\]
as defined in Section~\ref{ssec:measure_topological_network}.  Concretely:
\begin{itemize}
  \item \textbf{Geometry:} choose a symmetric kernel \(k(x,x')\) (e.g.\ Gaussian affinity or Euclidean distance) and uniform measure \(\mu(x_i)=1/N\).
  \item \textbf{Topology:} we use a powerful topological tool \textit{Ripserer.jl} to compute persistent homology in the desired dimension (e.g.\ 0D or 1D) to obtain a set of cycle representatives \(Y\).  Record each cycle’s birth–death coordinates via \(\iota:Y\to\Lambda\) and assign equal mass \(\nu(y)=1/|Y|\).
  \item \textbf{Incidence:} for each pair \((x_i,y)\), set \(\omega(x_i,y)=1\) if \(x_i\) lies on the chosen representative cycle \(y\), and \(0\) otherwise.
\end{itemize}
\revision{Analogously, for a second point cloud \(X'\), we construct the corresponding measure topological network \(P'=((X',k',\mu'),(Y',\iota',\nu'),\omega')\) following the same procedure.}

\subsection{TpOT Distance and Optimal Coupling}\label{ssec:tpot_compute}

    We then compute the TpOT distance of order \(p=2\) between \(P\) and \(P'\) (\ref{eq:tpot}) in which all geodesics in \(\mathcal{P}/\!\!\sim_w\) are convex.

    For clarity and efficient implementation, we now express each distortion in purely tensor‐ or matrix‐based summation form. \revision{We denote the cardinalities of the point clouds by \(N=|X|\) and \(N'=|X'|\), and the number of persistent homology generators (cycles) by \(M=|Y|\) and \(M'=|Y'|\). With these notations, let:} 
    \begin{align*}
    &L^{\mathrm{geom}}_{\,ii'jj'} 
    := \bigl|\,k(x_i,x_j)-k'(x'_{i'},x'_{j'})\bigr|^p,\quad
    1\le i,j\le N,\;1\le i',j'\le N',\\
    &L^{\mathrm{topo}}_{\,uv} 
    := \begin{cases}
         \|\iota(y_u)-\iota'(y'_v)\|^p,\quad \quad \quad&1\le u\le M,\;1\le v\le M',\\
         \|\iota(y_u)-\iota'(\partial_{Y'})\|^p, \quad &1\le u\le M, \; v=M'+1,\\
         \|\iota(\partial_{Y})-\iota'(y'_v)\|^p, \quad &u=M+1, \; 1\le v\le M',\\
         0, &u=M+1, v=M'+1
    \end{cases}
    \\
    &L^{\mathrm{hyper}}_{ii'\,uu'}
    :=
    \begin{cases}
    \tfrac12\,\bigl|\omega_{i,u}-\omega'_{i',u'}\bigr|^p,\quad \quad \quad
    &1\le u\le M,\;1\le u'\le M',\\[6pt]
    \tfrac12\,\bigl|\omega'_{i',u'}\bigr|^p,
    &u=M+1,\;1\le u'\le M',\\[6pt]
    \tfrac12\,\bigl|\omega_{i,u}\bigr|^p,
    &1\le u\le M,\;u'=M'+1,\\[6pt]
    0,
    &u=M+1,\;u'=M'+1.
    \end{cases}
    \end{align*}

    Let \(\Pi^v\in\mathbb{R}^{N\times N'}\) and \(\Pi^e\in\mathbb{R}^{(M+1)\times (M'+1)}\) be the optimal coupling matrices.  Then

    \begin{align}
    \label{eq:Lgeom_num}
    \mathcal{L}_{\mathrm{geom}}(\Pi^v)
    &=\sum_{i,i'=1}^{N,N'}\sum_{j,j'=1}^{N,N'} 
    L^{\mathrm{geom}}_{\,ii'jj'}\;\Pi^v_{\,ii'}\;\Pi^v_{\,jj'},\\
    \label{eq:Ltopo_num}
    \mathcal{L}_{\mathrm{topo}}(\Pi^e)
    &=\sum_{u=1}^{M+1}\sum_{v=1}^{M'+1} 
    L^{\mathrm{topo}}_{\,uv}\;\Pi^e_{\,uv},\\
    \label{eq:Lhyper_num}
    \mathcal{L}_{\mathrm{hyper}}(\Pi^v,\Pi^e)
    &=\sum_{i=1}^{N}\sum_{i'=1}^{N'}\sum_{u=1}^{M+1}\sum_{v=1}^{M'+1}
    L^{\mathrm{hyper}}_{\,ii'uu'}\;\Pi^v_{\,ii'}\;\Pi^e_{\,uu'}.
    \end{align}

    Given a 4-way tensor $L$ and a matrix $(C_{ij})_{ij}$, we define tensor-matrix multiplication\cite{chowdhury2020gromov}
    \begin{align*}
        L \otimes C := \left(\sum_{kl} L_{ijkl} C_{kl}\right)_{ij}.
    \end{align*}
    
    So the TpOT distance between two measure topological networks \(P\) and \(P'\) with tunable weights can be written as 
    \begin{align}
        L_{\mathrm{TpOT},p}(\Pi^v,\Pi^e,\alpha,\beta) &= \; \alpha\mathcal{L}_{\mathrm{geom}}(\Pi^v) \;+ \; (1-\alpha)\mathcal{L}_{\mathrm{topo}}(\Pi^e)\; +\; \beta\mathcal{L}_{\mathrm{hyper}}(\Pi^v,\Pi^e)\\
        &=\alpha \langle L^\mathrm{geom}\otimes \Pi^v,\Pi^v\rangle\;+\;(1-\alpha)\langle L^\mathrm{topo}, \Pi^e\rangle \;+\; \beta\langle L^\mathrm{hyper}\otimes \Pi^v,\Pi^e\rangle
    \end{align}

    In practice we solve the TpOT problem \eqref{eq:tpot} in its entropic‐regularised form  
    \[
    \min_{\substack{\pi^v\in\Pi(\mu,\mu')\\\pi^e\in\Pi_{\mathrm{adm}}(\nu,\nu')}}\;
    L_{\mathrm{TpOT},p}(\pi^v,\pi^e,\alpha,\beta)
    \;+\;\varepsilon_v\,\mathrm{KL}\bigl(\pi^v\mid\mu\otimes\mu'\bigr)
    \;+\;\varepsilon_e\,\mathrm{KL}\bigl(\pi^e\mid\nu\otimes\nu'\bigr).
    \]  
    Following section3.4 of \cite{zhang2025topological}, we solve the entropically regularised TpOT problem by alternately updating the two couplings while keeping the other fixed. These two Sinkhorn-style updates are utilized until convergence to \((\pi^v_\star,\pi^e_\star)\).  

    \subsection{Hypergraph Reconstruction and Dynamic Distortion}\label{ssec:geodesic_losses}
    \revision{Theoretically, the space of measure topological networks admits geodesics defined by convex combinations of their components~\cite{zhang2025topological}.}
    \revision{Given an optimal coupling}\((\pi^v_\star,\pi^e_\star)\), \revision{a theoretical} constant‐speed geodesic \(P_t\) for \(t\in[0,1]\) \revision{would be defined by linear interpolation:}
    \[
    P_t = \Bigl((X\times X',\,k_t,\,\pi^v_\star),\;((Y\times Y')\cup (Y\times \partial_{Y'})\cup (\partial_Y \times Y'),\,\iota_t,\,\pi^e_\star),\,\omega_t\Bigr),
    \]
    where linear interpolation yields
    \[
    k_t = (1-t)\,k + t\,k',\quad
    \iota_t = (1-t)\,\iota + t\,\iota',\quad
    \omega_t = (1-t)\,\omega + t\,\omega'.
    \]
    \revision{\textbf{Remark:} A direct linear interpolation of the incidence matrix \(\omega\) raises two fundamental issues: First, it could inevitably introduce fractional membership values (e.g., \(\omega_t(x,y) \in (0,1)\)), which are ill-defined in the context of binary hypergraphs. Second, and more critically, the topological structures generated by such abstract algebraic interpolation may become detached from the geometric reality of the data manifold; specifically, the interpolated cycles my fail to correspond to actual loops formed by the interpolated points.}\\

    \revision{To overcome these limitations and preserve physical validity, we propose a \textbf{hypergraph reconstruction} strategy. Instead of interpolating the parameters of measure topological networks directly, we fist interpolate the underlying geometry, next \textbf{re-compute} the topological structure and then reconstruct the measure topological network (Hypergraph). Specifically, with the optimal geometric coupling \(\pi^v_\star\), we construct the continuous trajectory as follows:

\begin{itemize}
    \item \textbf{Step 1: Geometric Interpolation:} Following the geodesic construction introduced by Han et al. \cite{han2025covariance,zhang2025topological}, we first obtain an optimal matching from the entropic regularized spatial coupling \(\pi^v_\star\). For the matched point pairs, we linearly interpolate their squared Euclidean distance matrices as \(k_t = (1-t)k + tk'\). The spatial coordinates of the intermediate point cloud \(X_t\) are then recovered by applying the Multidimensional Scaling (MDS) algorithm to \(k_t\), followed by a rigid transformation step for spatial alignment.
    \item \textbf{Step 2: Topological Structure Reconstruction:} Rather than artificially blending abstract features, we re-compute the persistent homology directly on the newly generated point cloud \(X_t\) to extract the authentic topological features (e.g., persistence barcodes) that physically exist at time \(t\).
    \item \textbf{Step 3: Measure Topological Network (Hypergraph) Reconstruction:} Using the spatial coordinates \(X_t\) and the authentically extracted topological features, we construct the intermediate measure topological network 
    \[
    P_t = \bigl((X_t, k_t, \mu_t), (Y_t, \iota_t, \nu_t), \omega_t\bigr)
    \] 
    following the pipeline outlined in Section~\ref{ssec:network_construction}.
    \item \textbf{Step 4: Early Warning Indicator with Dynamic Distortion:} Finally, the dynamic structural distortions \(\mathcal{L}^t\) are evaluated by solving the TpOT problem between the initial reference source \(P_0\) and the reconstructed intermediate network \(P_t\) at each time step \(t \in [0,1]\). That is, at each interpolation step \(t\), we measure the optimal discrepancy between the initial reference network \(P_0\) and the reconstructed hypergraph \(P_t\).
    The three dynamic distortion curves are defined as the minimized costs as follows.
\end{itemize}}

    \revision{
    Let \(L^{\mathrm{geom},t}\), \(L^{\mathrm{topo},t}\), and \(L^{\mathrm{hyper},t}\) denote the cost tensors defined between the static source \(P_0\) and the dynamic reconstruction \(P_t\). Unlike the theoretical linear interpolation where the coupling remains fixed, here we solve the optimal transport problem at each step \(t\). Let \((\pi^v_t, \pi^e_t)\) be the optimal couplings achieving the TpOT distance \(d_{\mathrm{TpOT}}(P_0, P_t)\). 
    
    \begin{align}
    \mathcal{L}_{\mathrm{geom}}^t
    &= \iint_{(X\times X_t)^2} \bigl|k(x,y) - k_t(x',y')\bigr|^p \,\mathrm{d}\pi^v_t(x,x')\,\mathrm{d}\pi^v_t(y,y') \label{eq:geom_curve}\\
    \mathcal{L}_{\mathrm{topo}}^t
    &= \int_{\bar Y\times\bar Y_t} \|\iota(y) - \iota_t(y')\|^p \,\mathrm{d}\pi^e_t(y,y')  \label{eq:topo_curve}\\
    \mathcal{L}_{\mathrm{hyper}}^t
    &= \int_{X\times X_t \times Y \times Y_t} \bigl|\omega(x,y) - \omega_t(x',y')\bigr|^p \,\mathrm{d}\pi^v_t(x,x')\,\mathrm{d}\pi^e_t(y,y') . 
    \label{eq:hyper_curve}
    \end{align}

    These curves \(\mathcal{L}^t\) quantitatively track the evolution of structural deviations.
    Specifically, \(\pi^v_t\) and \(\pi^e_t\) adaptively update the matching between the source and the evolving manifold, ensuring that the distortion reflects the true topological difference rather than artifacts of linear interpolation.

    Similarly, let the time-dependent cost tensors \(L^{\mathrm{geom},t}_{ii'jj'} ,L^{\mathrm{topo},t}_{uv} ,L^{\mathrm{hyper},t}_{ii'uv} \) be computed between the reference network \(P_0\) and the reconstructed network \(P_t\) (using the re-computed maps \(k_t\), \(\iota_t\), and \(\omega_t\)).
    Since the topological structure of \(P_t\) is re-generated at each step, the optimal correspondence may evolve.
    Let \(\Pi^v_t\) and \(\Pi^e_t\) denote the \emph{time-specific} optimal couplings obtained by solving the TpOT problem between \(P_0\) and \(P_t\).
    The three dynamic distortion curves are then calculated as the minimized transport costs:
    \begin{align}
    \mathcal{L}_{\mathrm{geom}}^t
    &= \sum_{i,i',j,j'} L^{\mathrm{geom},t}_{ii'jj'} \,\Pi^v_{t,\,ii'}\,\Pi^v_{t,\,jj'}
    = \bigl\langle L^{\mathrm{geom},t}\otimes\Pi^v_t,\,\Pi^v_t\bigr\rangle,
    \label{eq:geom_curve_discrete}\\
    \mathcal{L}_{\mathrm{topo}}^t
    &= \sum_{u,v} L^{\mathrm{topo},t}_{uv}\,\Pi^e_{t,\,uv}
    = \bigl\langle L^{\mathrm{topo},t},\,\Pi^e_t\bigr\rangle,
    \label{eq:topo_curve_discrete}\\
    \mathcal{L}_{\mathrm{hyper}}^t
    &= \sum_{i,i',u,v} L^{\mathrm{hyper},t}_{ii'uv}\,\Pi^v_{t,\,ii'}\,\Pi^e_{t,\,uv}
    = \bigl\langle L^{\mathrm{hyper},t}\otimes\Pi^e_t,\,\Pi^v_t\bigr\rangle.
    \label{eq:hyper_curve_discrete}
    \end{align}
    Figure \ref{fig:geodesic_interpolation} illustrates this reconstruction-based TpOT inetrpolation and dynamic evaluation framework.}

    These three \emph{dynamic distortion curves} \(\mathcal{L}^t\) capture the deformation of geometry, topology, and cycle incidence as \(t\) increases. 


    \subsection{Entropy as Early Warning Indicator}\label{ssec:entropy}

    To detect abrupt topological events along the geodesic, we compute two entropy metrics at each \(t\): persistence entropy and hypergraph entropy.

    First we apply the persistence entropy which is introduced in \cite{chintakunta2015entropy}. Extract the barcode \(B_t=\{[b_i^t,d_i^t]\}\) in \(\iota_t(Y,Y') \).  Let \(l_i^t=d_i^t-b_i^t\) and \(L_t=\sum_i l_i^t\).  Define the \textbf{Persistence Entropy} on geodesic as 
    \begin{align}
    \mathrm{PE}(B_t) = -\sum_i \frac{l_i^t}{L_t}\log\!\Bigl(\tfrac{l_i^t}{L_t}\Bigr).
    \end{align}
    Intuitively, entropy measures how different bars of the barcodes are in length. A barcode with uniform lengths has small entropy. Large changes in \(\mathrm{PE}(B_t)\) highlight the birth or death of significant homological features.

    In order to detect abrupt topological transitions along the TpOT geodesic, one may consider Shannon‐type entropies defined on the hypergraph.  \revision{Let \(H = (V, E)\) denote the hypergraph associated with a measure topological network, where \(V\) is the set of vertices (corresponding to the data points in \(X\)) and \(E\) is the set of hyperedges (corresponding to the persistent cycles in \(Y\)).} A classical proposal (e.g.\cite{bloch2019new}) proceeds by forming the matrix
    \[
    L(H) \;=\; I(H)\,I(H)^\mathsf{T}
    \]
    where \(I(H)\)is the incidence matrix and then computing the eigenvalues\(0 \leq\lambda_1(L(H))\leq\lambda_2(L(H))\leq\ldots\leq\lambda_m(L(H))\).  Then classicaly the Shannon entropy of hypergraph \(H\) is defined as
    \begin{align}
    \label{eq:shannon_entropy}
        S(H)=-\sum_{i=1}^m\mu_i\log_2(\mu_i),
    \end{align}
    where $\mu_i$ is defined as
    \begin{align*}
        \mu_i = \frac{\lambda_i(L(H))}{\sum_{i=1}^m \lambda_i(L(H))} = \frac{\lambda_i(L(H))}{\operatorname{Tr}(L(H))}.
    \end{align*}
    
    While this \emph{Shannon entropy} captures global connectivity patterns, it has two drawbacks:
    \begin{itemize}
      \item It requires eigen‐decomposition of an \(m\times m\) matrix (where \(m=|V|\)), which can be expensive for large hypergraphs.  
      \item It smooths over local vertex–edge redistributions, and thus may fail to react sharply to small but topologically significant changes (e.g.\ the intersection of several cycles).
    \end{itemize}

    \revision{To address these issues and formulate a mathematically rigorous measure of structural information, we adopt the information-theoretic framework for networks proposed by Dehmer (2008) \cite{dehmer2008novel,Dehmer2008InformationPI}. Dehmer's paradigm assigns a probability value to each graph element based on a local structural functional $f(\cdot)$, followed by the computation of Shannon entropy.

    Let $H = (V, E)$ be the hypergraph formed by the extracted persistent cycles, with incidence matrix $\omega \in \{0,1\}^{|V| \times |E|}$. Let $L(v) = \sum_{e \in E} \omega(v,e)$ be the degree of vertex $v$, and $S(e) = \sum_{v \in V} \omega(v,e)$ be the size of hyperedge $e$. 
    To rigorously avoid singularities such as $0 \cdot \ln(0)$ or division by zero, we restrict our analysis to the \emph{active vertex set} $V^* = \{v \in V \mid L(v) > 0\}$ and the \emph{active hyperedge set} $E^* = \{e \in E \mid S(e) > 0\}$. The total incidence of the hypergraph is $I_{total} = \sum_{v \in V^*} L(v) = \sum_{e \in E^*} S(e)$.

    }
    \revision{
    \begin{definition}[\textbf{Vertex-Perspective Entropy}]\label{def:HEV}
        Following the Dehmer framework, we define the structural functional for a vertex as its degree, $f(v) = L(v)$. Normalizing this functional over the active set yields a valid probability distribution $p(v) = \frac{L(v)}{I_{total}}$, which represents the probability that a randomly chosen incidence connection belongs to vertex $v$. The Vertex-Perspective Entropy is defined as:
        \begin{equation}
            \mathrm{HE}_V(H) = -\sum_{v \in V^*} p(v) \ln p(v).
        \end{equation}
    \end{definition}

    \begin{definition}[\textbf{Hyperedge-Perspective Entropy}]\label{def:HEE}
        Dually, defining the structural functional for a hyperedge as its size $f(e) = S(e)$ yields the probability distribution $q(e) = \frac{S(e)}{I_{total}}$. The Hyperedge-Perspective Entropy is defined as:
        \begin{equation}
            \mathrm{HE}_E(H) = -\sum_{e \in E^*} q(e) \ln q(e).
        \end{equation}
    \end{definition}

    By formulating the entropies as standard Shannon entropies over the discrete probability spaces $V^*$ and $E^*$, we obtain the following rigorous structural properties:
    }
    \revision{

    \begin{property}[\textbf{Bounds and Maximal Assumption}]\label{thm:maximal}
        The entropies are bounded by $0 \le \mathrm{HE}_V(H) \le \ln |V^*|$ and $0 \le \mathrm{HE}_E(H) \le \ln |E^*|$. 
        The maximum $\mathrm{HE}_V(H) = \ln |V^*|$ is achieved if and only if $H$ is a \emph{regular hypergraph} (i.e., $L(v)$ is constant for all $v \in V^*$). 
        Dually, $\mathrm{HE}_E(H) = \ln |E^*|$ is achieved if and only if $H$ is a \emph{uniform hypergraph} (i.e., $S(e)$ is constant for all $e \in E^*$).
    \end{property}
    Proofs of this property and subsequent theorems are provided in the Appendix \ref{app:proof}.
    
    With this property, we could normalize $\widetilde{\mathrm{HE}_V}(H) = \frac{\mathrm{HE}_V(H)}{\ln |V^*|}$, $\widetilde{\mathrm{HE}_E}(H) = \frac{\mathrm{HE}_E(H)}{\ln |E^*|}\in [0,1]$.
    }
    \revision{

    \begin{definition}[\textbf{Symmetric Hypergraph Entropy}]
        To capture both perspectives simultaneously, we introduce \begin{equation}\label{def:HEsym}
        \mathrm{HE}_{\mathrm{sym}}(G)
        =\alpha\,\widetilde{\mathrm{HE}_V}
        +(1-\alpha)\,\widetilde{\mathrm{HE}_E},
        \quad
        \alpha\in[0,1].
        \end{equation}
        By adjusting \(\alpha\), one can emphasize vertex‐level (\(\alpha\approx1\)) or hyperedge‐level (\(\alpha\approx0\)) changes, or treat both equally (\(\alpha=0.5\)).
    \end{definition}

    \begin{theorem}[\textbf{Sensitivity to Abrupt Topological Transitions}]\label{thm:abrupt_change}
        Let $\{H_t\}_{t \in \mathcal{T}}$ be a sequence of dynamic hypergraphs parameterized by $t$. Suppose at a critical parameter $t_c$, an abrupt topological transition occurs via the emergence of a new active hyperedge $e_{new}$ of size $k > 0$. Then the vertex-perspective hypergraph entropy strictly changes at $t_c$ (i.e., $\lim_{t \to t_c^-} \mathrm{HE}_V(H_t) \neq \mathrm{HE}_V(H^+)$), except possibly for a highly restrictive set of hypergraphs whose degree sequences satisfy a specific, rigid Diophantine equation.
    \end{theorem}
    
    }
    \revision{

    \begin{theorem}[\textbf{Dual Sensitivity to Topological Transitions}]\label{thm:dual_change}
        The same conclusion holds for hyperedge-perspective entropy $\mathrm{HE}_E$, except for a highly restrictive zero-measure algebraic condition governed by the Fundemental Theorem of Arithmetic.
    \end{theorem}

    }
    \revision{

    \begin{theorem}[\textbf{Isomorphism Invariance of Topological Entropy}]\label{thm:iso}
        Let \(H_1 = (V_1, E_1)\) and \(H_2 = (V_2, E_2)\) be two hypergraphs derived from two measure topological networks. If \(H_1\) and \(H_2\) are isomorphic (i.e., there exist bijections \(\phi: V_1 \to V_2\) and \(\psi: E_1 \to E_2\) preserving the incidence relations \(\omega_1(v, e) = \omega_2(\phi(v), \psi(e))\)), then
        \[
        \mathrm{HE}_V(H_1) = \mathrm{HE}_V(H_2) \quad \text{and} \quad \mathrm{HE}_E(H_1) = \mathrm{HE}_E(H_2).
        \]
    \end{theorem}

    \begin{theorem}[\textbf{Algebraic Topological Upper Bound}]\label{thm:upper_bound}
        Let \(D_k\) denote the \(k\)-th persistence diagram obtained from the Vietoris–Rips filtration of the network, and let $|D_k|$ be the number of persistent generators (i.e., topological features with non-zero persistence). And suppose the hypergraph \(H\) is constructed using all persistent generators across all dimensions. Then the hyperedge-perspective entropy is strictly bounded by the topological complexity of the manifold:
        \[
        \mathrm{HE}_E(H) \leq \ln\left(\sum_{k} |D_k|\right).
        \]
    \end{theorem}

    }
    \revision{

    In many practical data analysis scenarios, the measure topological network is constructed explicitly focusing on a specific homological dimension $k$ (e.g., $k=1$ to analyze loops, or $k=2$ for voids). 
    \begin{corollary}[\textbf{Dimension-Specific Topological Bound}]
        If the hypergraph $H$ is constructed exclusively using the $k$-dimensional persistent generators, the hyperedge-perspective entropy is strictly bounded by the number of k-dimensional topological features :
        \begin{equation}
            \mathrm{HE}_E(H) \le \ln(|D_k|).
        \end{equation}
    \end{corollary}
    
    \textbf{Remark:} 
    This corollary has significant practical implications. For instance, when tracking the structural dynamics of a 1-dimensional functional network, the maximum possible hyperedge entropy is fundamentally bottlenecked by the total number of k-dimensional persistent features $\ln(|D_1|)$. This implies that our entropy metric is not only a statistical measure of incidence, but a direct proxy for the 1-dimensional algebraic topological capacity of the system.    
    }

    Our newly proposed hypergraph entropies have following advantages 
    \begin{itemize}
      \item \emph{Computational efficiency:} each requires only \(O(|V|\cdot|E|)\) operations, avoiding costly eigen‐decompositions.
      \item \revision{\emph{Theoretical Sensitivity:} As proven in Theorem 1, the abrupt appearance or disappearance of a topological feature instantaneously forces a non-zero displacement in the probability simplex. This guarantees a mathematically rigorous entropy spike or inflection, making it highly sensitive to critical topological events.}
      \item \revision{\emph{Structural Interpretability:} The metrics provide a clear macroscopic interpretation of topological uniformity. Furthermore, their theoretical upper bounds are deeply rooted in algebraic topology: the hyperedge entropy is fundamentally bottlenecked by the topological capacity of the system (i.e., bounded by the p ersistent features, \(\mathrm{HE}_E \le \ln \beta_k\), as established in Theorem 3 and Corollary 1), serving as a direct quantitative proxy for the homological complexity of the data manifold.}
    \end{itemize}

    We compute\(\mathrm{HE}_{\mathrm{sym}}(t)\) on the interpolated hypergraph \(G_t\) along the TpOT geodesic, and use sudden deviations in these curves to flag discrete topological events.  Sudden deviations in \(\mathrm{HE}_{\mathrm{sym}}\) mark structural reorganizations of the cycle hypergraph.
    \revision{In practice, we evaluate the symmetric hypergraph entropy \(\mathrm{HE}_{\mathrm{sym}}(H_t) = \alpha \frac{\mathrm{HE}_V(H_t)}{\ln|V^*|} + (1-\alpha) \frac{\mathrm{HE}_E(H_t)}{\ln|E^*|}\) on the dynamically reconstructed hypergraph \(H_t\) at each interpolation step. We then use the sharp discontinuities and sudden deviations in these curves to detect discrete topological phase transitions and structural reorganizations within the evolving point cloud. For clarity, we provide pseudocode below for our method (Algorithm \ref{alg:overall}).} 

    \subsection{Point-Level Hypergraph Entropy Change}
\label{ssec:point-level-entropy}
    Complex hypergraph networks often contain multiple interdependent cycles that encode structural relationships at different spatial and topological scales. When such networks evolve---for instance, under temporal deformation, diffusion, or connectivity reorganization---the \emph{global} hypergraph entropy tends to average out the local transformations, thereby obscuring where the most significant structural changes occur. To resolve this limitation, we introduce a \textit{cycle-level entropy decomposition} that projects entropy down to the vertex level. This formulation enables identification of local topological variations within complex hypergraph systems.
    
\revision{
\paragraph{Incidence transport via optimal coupling.}
 We define this decomposition generally between a \emph{reference} hypergraph (with incidence matrix $\omega \in \mathbb{R}_{\ge 0}^{n \times m}$) and a \emph{target} hypergraph (with incidence matrix $\omega' \in \mathbb{R}_{\ge 0}^{n' \times m'}$). Here, $\omega[i,j]$ indicates the participation weight of vertex $i$ in the $j$-th reference cycle. Because the ordering of topological features generally differs between states, the target matrix $\omega'$ is column-aligned to the reference $\omega$ using a permutation (or assignment) matrix $A \in \{0, 1\}^{m' \times m}$:
\[
    \widehat{\omega} = \omega' A.
\]
This matrix multiplication explicitly permutes the columns of the target incidence matrix, yielding a column-aligned matrix $\widehat{\omega} \in \mathbb{R}_{\ge 0}^{n' \times m}$ where the $j$-th column directly corresponds to the $j$-th reference cycle.

\paragraph{Column normalization.}
To compute the entropy, the columns of the incidence matrices are normalized to form probability distributions over the vertices. For the reference matrix $\omega$ and the aligned target matrix $\widehat{\omega}$, we explicitly define their normalized counterparts $P$ and $\widehat{P}$ as:
\[
    P[i,j] = \frac{\omega[i,j]}{\sum_{i'=1}^{n}\omega[i',j]+\varepsilon}, 
    \qquad 
    \widehat{P}[i,j] = \frac{\widehat{\omega}[i,j]}{\sum_{i'=1}^{n'}\widehat{\omega}[i',j]+\varepsilon},
\]
where $\varepsilon$ is a sufficiently small constant to ensure numerical stability.

\paragraph{Cycle-level entropy.}
Let $p_j = P[:,j]$ denote the normalized vertex distribution of the $j$-th reference cycle, and $\widehat{p}_j = \widehat{P}[:,j]$ denote the corresponding distribution from the aligned target matrix. We define the cycle-level entropy for the reference and target states, respectively, as:
\begin{equation}
    H_j = -\frac{1}{\log n} \sum_{i=1}^{n} p_{ij}\,\log p_{ij}, 
    \qquad 
    \widehat{H}_j = -\frac{1}{\log n'} \sum_{i=1}^{n'} \widehat{p}_{ij}\,\log \widehat{p}_{ij},
    \label{eq:cycle_entropy}
\end{equation}
which measures how spatially diffuse (large entropy) or concentrated (small entropy) the participation is within that specific topological cycle. 

\paragraph{Entropy difference and propagation.}
For each reference cycle $j$, the change in structural entropy between the two states is defined as:
\begin{equation}
    \Delta H_j = \widehat{H}_j - H_j.
\end{equation}
To localize this structural variation onto the target vertex set, the absolute entropy change $|\Delta H_j|$ is propagated back to the target vertices according to their membership weights in the aligned cycles:
\begin{equation}
    s_i = \sum_{j=1}^{m} \widehat{P}[i,j]\, |\Delta H_j|.
\label{eq:point_score}
\end{equation}
By defining the column vector $s = [s_1, \dots, s_{n'}]^\top$ and the entropy difference vector $\Delta H = [\Delta H_1, \dots, \Delta H_m]^\top$, this back-projection can be compactly written in matrix form as $s = \widehat{P}\,|\Delta H|$. The vector $s$ forms a \emph{vertex-level hypergraph-entropy field}, which highlights localized regions undergoing the strongest structural transformations. 

\paragraph{Interpretation.}
Equations~\eqref{eq:cycle_entropy}–\eqref{eq:point_score} transform the complex structural comparison of high-dimensional hypergraphs into an interpretable scalar field over the vertex domain. Visualizing the field $s$ on the spatial point cloud reveals localized topological deformations that global entropy measures inherently average out.
}

\begin{algorithm}[htbp]
\caption{\revision{Early warning detection via dynamic distortions and Entropy}}\label{alg:overall}
\begin{algorithmic}[1]

\revision{
\Require Time-series point clouds $\{X^{(i)}\}_{i=1}^T$, distance kernel $k$.
\Require TpOT trade-off weights $(\alpha,\beta)$, regularization weights $(\varepsilon_v, \varepsilon_e)$, and symmetric entropy weight $\gamma$.
\Require Local geodesic resolution $L$, forming a uniform partition of the interval $[0,1]$ denoted by $\{\tau_\ell\}_{\ell=1}^L$.
\Ensure A continuous global trajectory of structural evaluations: distortion curves $\mathcal{L}_{\mathrm{geom}}(\tau^*), \mathcal{L}_{\mathrm{topo}}(\tau^*), \mathcal{L}_{\mathrm{hyper}}(\tau^*)$ and entropy curves $\mathrm{PE}(\tau^*), \mathrm{HE}_{\mathrm{sym}}(\tau^*)$ parameterized by the global continuous timeline $\tau^* \in [0, 1]$.

\State \textbf{Initialization:} 
\State Construct the global reference network $P^{(1)} = \bigl((X^{(1)},k^{(1)},\mu^{(1)}),\,(Y^{(1)},\iota^{(1)},\nu_1),\,\omega^{(1)}\bigr)$.
\Statex

\For{$i = 1$ \textbf{to} $T-1$}
    \State Construct networks $P^{(i)}$ and $P^{(i+1)}$ via persistent homology.
    \State Compute TpOT between $P^{(i)}$ and $P^{(i+1)}$ using weights $(\alpha, \beta)$ and regularization weights $(\varepsilon_v, \varepsilon_e)$ to obtain the optimal spatial coupling $\pi^v_\star$.
    \Statex
    
    \For{$\ell = 1$ \textbf{to} $L$}
        \State $\tau \leftarrow \tau_\ell$ \quad \emph{// Local geodesic parameter $\tau \in [0,1]$}
        \State $\tau^* \leftarrow (i + \tau - 1)/(T-1)$ \quad \emph{// Mapping to global continuous timeline $\tau^* \in [0, 1]$}
        
        \State \emph{// Step 1: Geometric Displacement Interpolation}
        \State Compute interpolated spatial positions based on matched pairs $(x, x') \sim \pi^v_\star$:
        \State \quad $X_{\tau^*} \leftarrow \operatorname{Embed}\Bigl( \bigl\{ k_{\tau^*} = (1-\tau)k^{(i)} + \tau k^{(i+1)} \bigm| \pi^v_\star > 0 \bigr\} \Bigr)$
        
        \State \emph{// Step 2: Topological Reconstruction}
        \State Re-compute persistent homology strictly on $X_{\tau^*}$ to construct the intermediate network $P_{\tau^*}$.

        \State \emph{// Step 3: Dynamic Curve Evaluation (Relative to Global Reference $P^{(1)}$)}
        \State Solve TpOT between $P^{(1)}$ and $P_{\tau^*}$ using weights $(\alpha, \beta, \varepsilon_v, \varepsilon_e)$ to obtain evaluation couplings $(\Pi^v_{\tau^*}, \Pi^e_{\tau^*})$.
        \State Obtain dynamic distortion curves by computing the tensor inner products:
        \State \quad $\mathcal{L}_{\mathrm{geom}}(\tau^*) = \bigl\langle L_{\mathrm{geom}}^{\tau^*},\,\Pi^v_{\tau^*}\otimes\Pi^v_{\tau^*}\bigr\rangle$
        \State \quad $\mathcal{L}_{\mathrm{topo}}(\tau^*) = \bigl\langle L_{\mathrm{topo}}^{\tau^*},\,\Pi^e_{\tau^*}\bigr\rangle$
        \State \quad $\mathcal{L}_{\mathrm{hyper}}(\tau^*) = \bigl\langle L_{\mathrm{hyper}}^{\tau^*},\,\Pi^v_{\tau^*}\otimes\Pi^e_{\tau^*}\bigr\rangle$
        
        \State Obtain structural entropy curves from the active topology of $P_{\tau^*}$:
        \State \quad $\mathrm{PE}(\tau^*) = -\sum \frac{l_j}{L} \log \frac{l_j}{L} \quad$ \emph{// from barcode lengths of $Y_{\tau^*}$}
        
        \State \quad $\mathrm{HE}_{\mathrm{sym}}(\tau^*) \leftarrow \mathrm{HE}_{\mathrm{sym}}(P_{\tau^*}; \gamma)$ \emph{// Evaluated with entropy weight $\gamma$}
        \State \quad Compute $\mathrm{HE}_V(\tau^*)$ and $\mathrm{HE}_E(\tau^*)$ identically following definition equations.
    \EndFor
\EndFor
}
\end{algorithmic}
\end{algorithm}

    This completes the detailed description of our method.  In Section~\ref{sec:experiments} we validate its effectiveness on synthetic and real‐world datasets.

    \section{Experiments}\label{sec:experiments}

    We evaluate our method \revision{on four distinct settings: two synthetic phenomenological bifurcation models, a high-dimensional biological aggregation model (D'Orsogna), and a real-world longitudinal fMRI dataset.} In each case, we report: (1) the three dynamic distortion curves \revision{($\mathcal{L}_{\mathrm{geom}}$, $\mathcal{L}_{\mathrm{topo}}$, $\mathcal{L}_{\mathrm{hyper}}$), which track the hierarchical structural evolution,} and (2) the entropy indicators, namely Persistence Entropy (PE), Symmetric Hypergraph Entropy ($\mathrm{HE}_{\mathrm{sym}}$), Vertex-Perspective Entropy($\mathrm{HE}_V$) and Hyperedge-Perspective Entropy($\mathrm{HE}_E$). These entropy measures provide complementary, parallel characterizations of topological evolution: PE captures variations in the persistence-diagram domain, while $\mathrm{HE}_{\mathrm{sym}},\mathrm{HE}_V,\mathrm{HE}_E$ quantify changes in the higher-order hypergraph incidence structure. In addition, the point-level hypergraph-entropy field is used to capture local vertices that contribute most to the observed structural reorganization in the real fMRI data.

    \revision{
    Across all experiments, the entropic TpOT problem is consistently optimized using the Sinkhorn algorithm. We configure the optimization with trade-off parameters $\alpha = 0.5$ and $\beta = 1.0$, along with entropic regularization weights $\varepsilon_s = 0.003$ for the spatial coupling and $\varepsilon_f = 0.01$ for the feature coupling. The symmetric entropy weight is set to $\gamma = 0.5$.
    }

    \revision{\subsection{Experiment 1: Topological Phase Transition in Stochastic Oscillators}

    \textbf{Data Generation.} 
    To evaluate the proposed dynamic TpOT and structural entropy framework on systems exhibiting complex topological phase transitions, we generate a synthetic dataset based on the stochastic Rayleigh-Van der Pol (RVP) oscillator \cite{tanweer2024topological}. Forced by additive white Gaussian noise, the stationary joint probability density function (PDF) of the RVP oscillator's state space $(x_1, x_2)$ is proportional to the exponential of its potential energy:
    \begin{equation}
        p(x_1, x_2) \propto \exp\left( -\frac{V(x_1, x_2)}{T} \right), \quad \text{where } V(x_1, x_2) = \frac{1}{2}(x_1^2+x_2^2)^2 + h(x_1^2+x_2^2),
    \end{equation}
    where $h$ is the critical bifurcation parameter and $T$ controls the effective noise intensity of the system. 

    As theoretically established in stochastic dynamical systems \cite{tanweer2024topological}, this oscillator undergoes a phenomenological bifurcation (P-bifurcation) exactly at $h=0$. For $h < 0$, the system exhibits limit-cycle oscillations (LCO), and its PDF forms a crater-like geometry. Topologically, this corresponds to a prominent 1-dimensional persistent loop (i.e., a Betti-1 feature). As $h$ increases past $0$, the system shifts to a monostable state, and the geometry structurally collapses into a single dense 0-dimensional connected component.

    To replicate a real-world discrete data acquisition scenario, we generated a time-varying sequence of point clouds by sampling from this analytical PDF. Since direct sampling from this unnormalized distribution is non-trivial, we employed the Metropolis-Hastings Markov Chain Monte Carlo (MCMC) algorithm. We simulated the dynamical evolution by discretizing the bifurcation parameter $h \in [-1, 1]$ into $51$ uniformly spaced snapshots. For each snapshot, we set the effective noise $T=0.001$ and extracted $N=200$ samples. 

     This rigorous procedure yields a dynamic point cloud sequence $\{X^{(i)}\}_{i=1}^{51}$ that accurately captures both the continuous geometric deformation and the abrupt topological phase transition from a ring to a single cluster, serving as the ground truth for our dynamic evaluation framework. A visualization of the sampled point clouds is showed in the top row of figure \ref{fig:rvp_evolution}.}

     \begin{figure}[H]
        \centering
        \includegraphics[width=\linewidth]{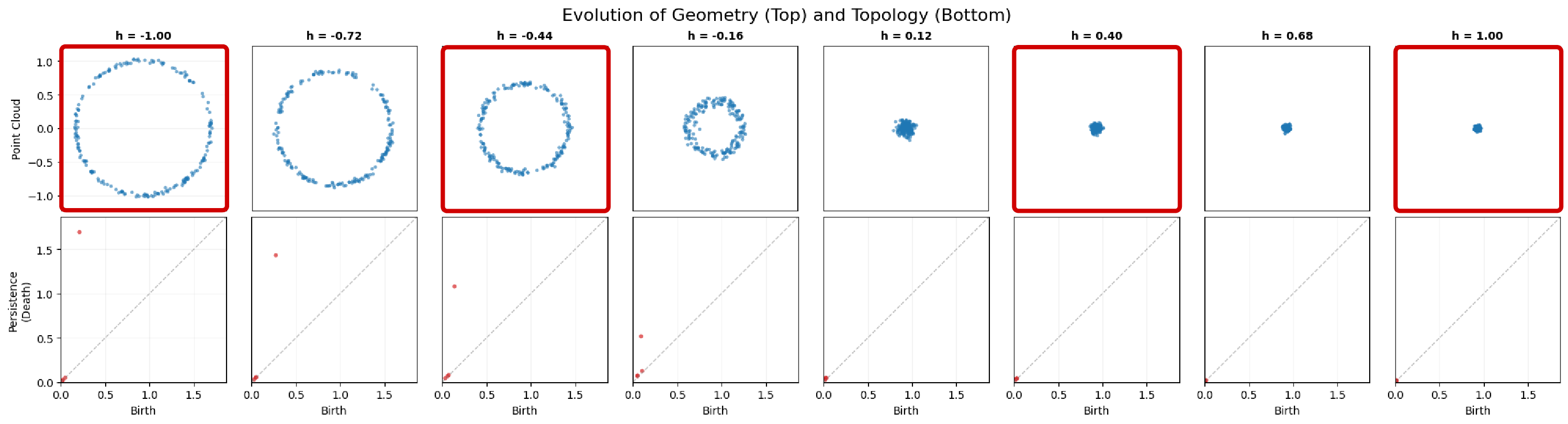} 
        \caption{\revision{Ground truth evolution of the stochastic RVP oscillator. (Top) Scatter plots of the state space showing the transition from a limit cycle to a monostable point. \textbf{Four sparse snapshots} are chosen as our training density samples for topological optimal transport task (illustrated in \textbf{red boxes}). (Bottom) The corresponding persistence diagrams tracking the birth and death of the 1-dimensional homological feature.}}
        \label{fig:rvp_evolution}
    \end{figure}
    
    \begin{figure}[H]
        \centering
        \includegraphics[width=0.8\linewidth]{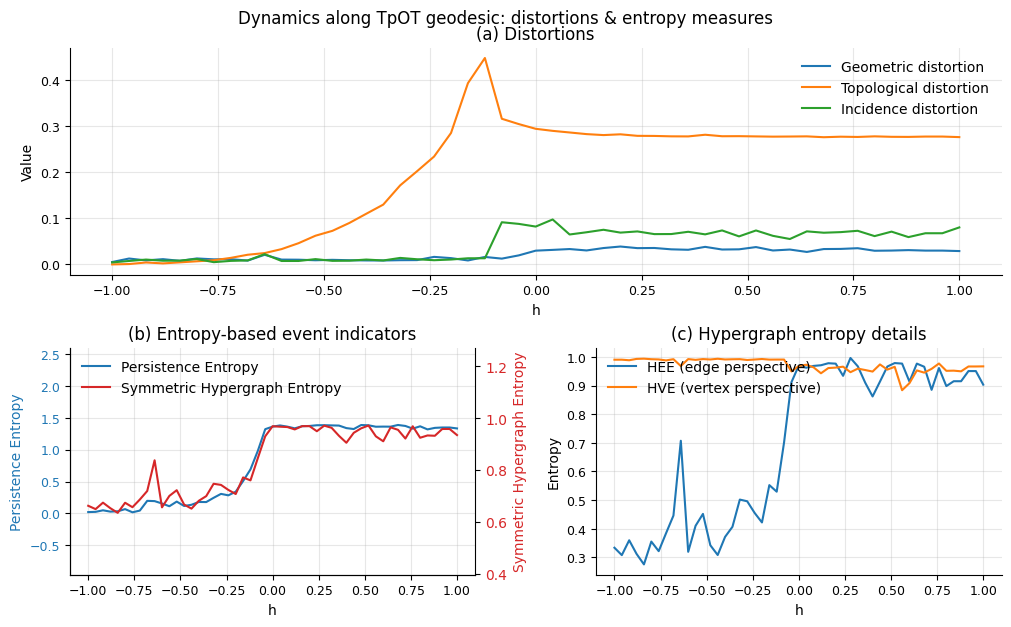} 
         \caption{\revision{Baseline(ground truth) sequential evaluation directly computed between the reference state ($h=-1$) and subsequent empirical snapshots. (a) The topological distortion peaks and flattens exactly as the limit cycle collapses. (b) Both persistence entropy and the proposed symmetric hypergraph entropy exhibit a discontinuous jump at the critical point $h=0$. (c) The jump in symmetric entropy is primarily driven by the hyperedge-perspective component ($\mathrm{HE}_E$).To facilitate a direct visual comparison, both perspective entropies are normalized by their respective theoretical upper bounds.}}
        \label{fig:rvp_distortions}
    \end{figure}

     \revision{
     \paragraph{Baseline Sequential Evaluation}

    Before evaluating our proposed geodesic interpolation method, we first apply our framework directly to the fully sampled empirical sequence over the parameter range $h \in [-1, 1]$ to establish a ground truth baseline as presented in figure \ref{fig:rvp_distortions}. As the bifurcation parameter $h$ increases, the system strictly follows the theoretical dynamics of the RVP oscillator: the crater-like limit cycle collapses into a dense, monostable cluster at the critical point $h=0$. To quantify this P-bifurcation, we compute the structural distortions relative to the global reference state $h=-1$.

    The distortion dynamics show a multi-scale temporal decoupling, corresponding to a 'macro-meso-micro' structural relaxation sequence. 
    At the \textbf{macro-scale}, the topological distortion ($\mathcal{L}_{\mathrm{topo}}$) peaks drastically just before $h=0$ as the global limit cycle contracts, incurring a substantial Wasserstein penalty. Once the macroscopic loop vanishes, $\mathcal{L}_{\mathrm{topo}}$
    flattens into a constant plateau, as the reference loop can only be matched to diagonal features. 
    Subsequently, at the \textbf{meso-scale}, the incidence distortion ($\mathcal{L}_{\mathrm{hyper}}$) peaks and reorganizes. The membership incidence matrix must dissolve and reorganize to accommodate the collapsed topology. Finally, at the \textbf{micro-scale}, the geometric distortion ($\mathcal{L}_{\mathrm{geom}}$) settles last, reflecting the physical diffusion time required for individual stochastic particles to cluster at the new local density peak. Accompanying this multi-scale cascade, the structural entropies (Persistence Entropy and the proposed $\mathrm{HE}_{\mathrm{sym}}$) exhibit a surge near the topological tipping point (h=0). This behavior suggests their potential utility as indicators of the phase transition without relying on predefined thresholds.
    }

    \begin{figure}[htbp]
        \centering
        \begin{subfigure}[b]{0.8\textwidth}
            \centering
            \includegraphics[width=\textwidth]{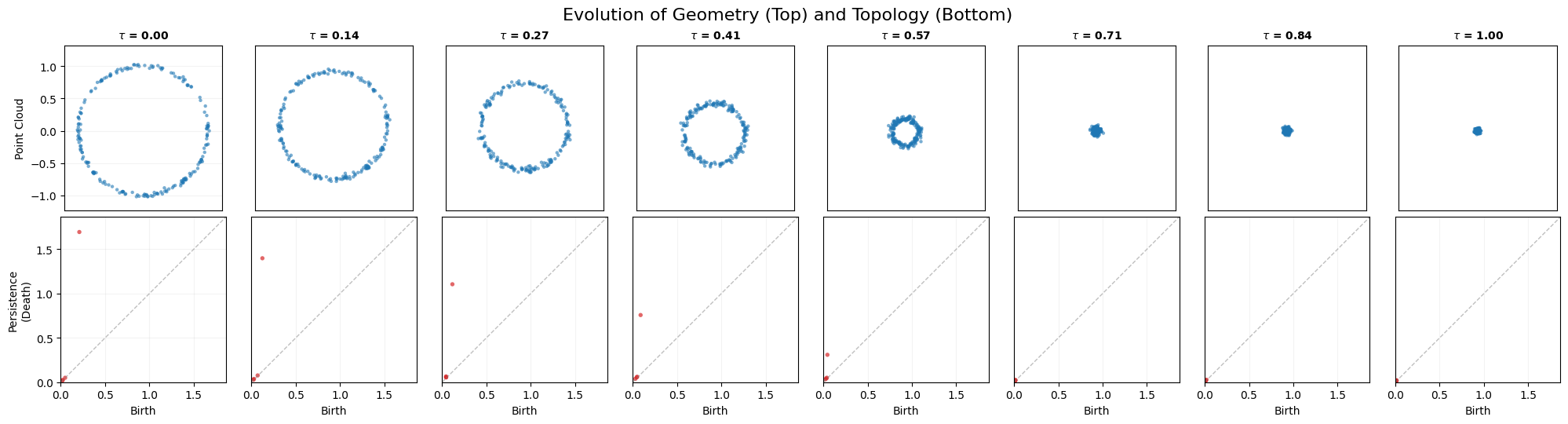}
            \caption{\revision{Reconstructed spatial geometry and corresponding persistence diagrams}}
            \label{fig:interp_a}
        \end{subfigure}
        \hfill 
        \begin{subfigure}[b]{0.8\textwidth}
            \centering
            \includegraphics[width=\textwidth]{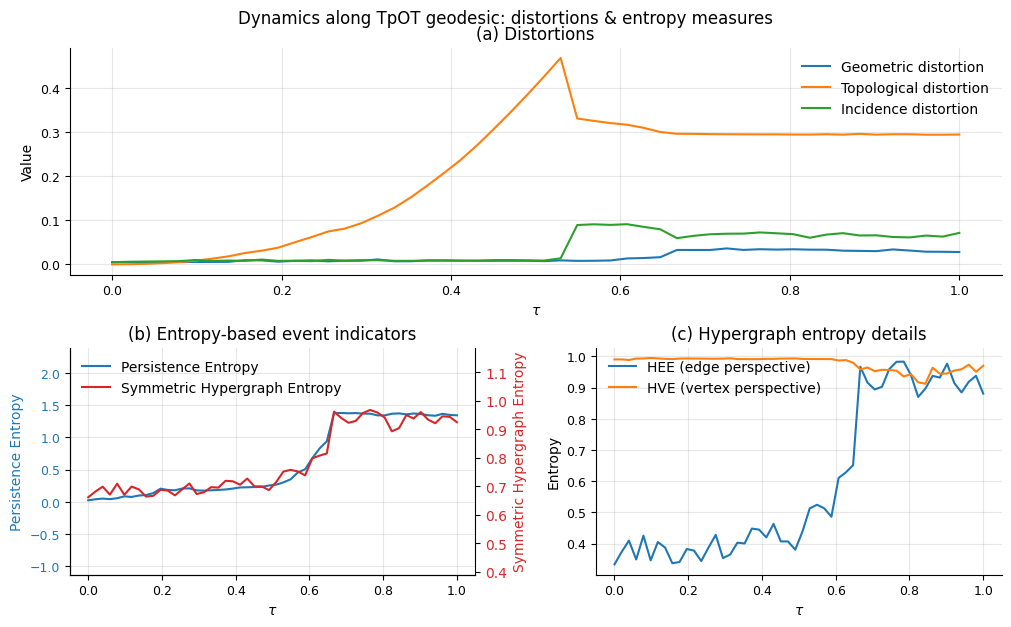}
            \caption{\revision{Reconstructed structural dynamics}}
            \label{fig:interp_b}
        \end{subfigure}
        
        \caption{\revision{\textbf{Experimental validation via TpOT geodesic interpolation.} We subsampled the empirical dataset into merely four equidistant keyframes and reconstructed the continuous topological evolution parameterized by $\tau \in [0,1]$. \textbf{(a)} The MDS-based isometric embedding interpolates the intermediate spatial geometries between the sparse keyframes, approximating the collapse of the limit cycle. \textbf{(b)} The dynamic evaluation along the reconstructed geodesic reproduces the hierarchical macro-meso-micro distortion sequence and the abrupt entropic jumps at the critical transition point, closely aligning with the unobserved ground truth dynamics.}}
        \label{fig:ex1_interpolation_validation}
    \end{figure}

    \revision{
    \paragraph{Dynamic Reconstruction of Topological Phase Transitions.}
    To demonstrate the predictive power of our methodology on temporally sparse observations, we artificially subsampled the dataset into merely four equidistant keyframes across the parameter space (simulating a low-resolution data acquisition scenario). We then applied our reconstruction-based TpOT interpolation framework (Algorithm \ref{alg:overall}) to generate the continuous trajectory parameterized by $\tau \in [0,1]$. The results are illustrated in figure \ref{fig:ex1_interpolation_validation}.

    The dynamic evaluation along the reconstructed geodesic trajectory approximates the unobserved P-bifurcation, closely aligning with the unobserved ground truth dynamics. As visualized in the interpolation results, the generated intermediate states $P_\tau$ capture both the geometric collapse and the progressive deterioration of the Betti-1 barcode. More importantly, the reconstructed distortion and entropy curves along the interpolation parameter $\tau$ closely aligning with the unobserved ground truth dynamics previously observed along the true parameter $h$. 

    Our method successfully reproduces the abrupt entropic jumps at the critical transition phase, strictly verifying the theoretical guarantees of Theorem 1(2) on interpolated data. Furthermore, the generated trajectory explicitly preserves the macro-meso-micro temporal decoupling: the interpolated $\mathcal{L}_{\mathrm{topo}}$ curve peaks and flattens well before the stabilization of the geometric component $\mathcal{L}_{\mathrm{geom}}$. This confirms that our interpolation strategy does not merely blend coordinates linearly, but rigorously reconstructs the authentic, multi-scale process of the underlying stochastic dynamical system from highly sparse observations.
     }

    \revision{
    \subsection{Experiment 2: Dimension-Specificity and Negative Control in a Double-Well Potential}
    
    \paragraph{Data Generation and Physical Model.} 
   
    \begin{figure}[H]
        \centering
        \includegraphics[width=\linewidth]{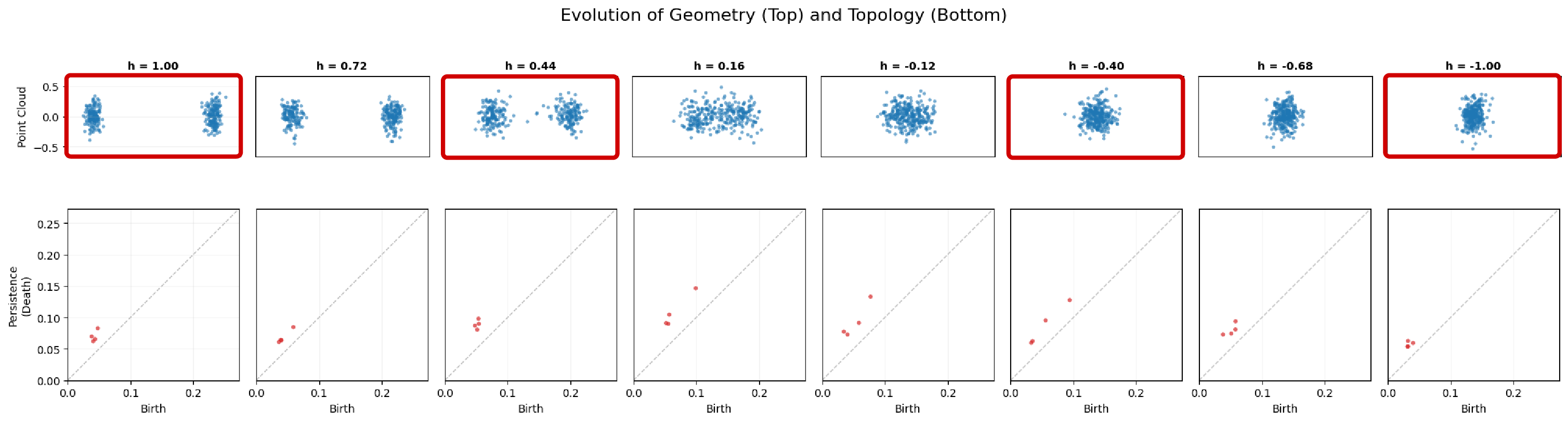} 
        \caption{\revision{Phenomenological bifurcation in the double-well potential model. (Top) Scatter plots of the state space showing the transition from a bistable regime(two distinct clusters) to a monostable regime(a single fused cluster). \textbf{Four sparse snapshots} are chosen as our training density samples for topological optimal transport task (illustrated in \textbf{red boxes}). (Bottom) The corresponding persistence diagrams tracking the birth and death of the 1-dimensional
    homological feature. }}
        \label{fig:double_well_scatter}
    \end{figure}
    
    To further rigorously evaluate our framework, particularly its dimension-specificity and the decoupling of its distortion components, we construct a second synthetic dataset inspired by the phenomenological bifurcations discussed in Appendix A of Tanweer et al. \cite{tanweer2024topological}. We simulate a stochastic system governed by a parameterized double-well potential:
    \begin{equation}
        V(x_1, x_2; h) = (x_1^2 - h)^2 + x_2^2,
    \end{equation}
    where the state space is sampled using a Metropolis-Hastings MCMC sampler with a generalized noise temperature $T=0.04$. The parameter $h$ decreases from $1.00$ to $-1.00$ over $51$ uniform snapshots. 
    
    As visualized in the scatter plots of Figure \ref{fig:double_well_scatter}, when $h > 0$, the system is in a \emph{bistable regime}, forming two distinct connected components ($\beta_0 = 2$). As $h \le 0$, the two wells merge into a single global minimum, and the point cloud topologically fuses into a single \emph{monostable} cluster ($\beta_0 = 1$). Unlike the RVP oscillator in Experiment 1, this fusion process strictly involves a 0-dimensional topological transition, completely devoid of any macroscopic 1-dimensional homological features (i.e., no authentic loops or $\beta_1$ generators are created or destroyed).

    \paragraph{Baseline Evaluation: Specificity and Dimension-Selectivity.} 
    To demonstrate the targeted selectivity of our proposed metrics, we explicitly constrained the persistent homology feature extraction exclusively to 1-dimensional homology ($H_1$) while evaluating this $\beta_0$-driven dataset. The fully separated state at $h=1.00$ serves as the global reference $P^{(1)}$.
    
    The baseline dynamics (computed directly on the full sequence)  support the specificity of our framework, as shown in the ground truth curves of Figure \ref{fig:double_well_curves_gt}. Because the two distinct probability masses physically migrate and converge toward the center over time, the pure geometric distortion ($\mathcal{L}_{\mathrm{geom}}$) exhibits a continuous, monotonic rise, successfully capturing the macroscopic spatial fusion. 
    
    However, since our topological hypergraph was strictly configured to monitor $H_1$ features, it acts as a precise theoretical filter. Because no macroscopic loops exist during the cluster merging, the $H_1$-based topological distortion ($\mathcal{L}_{\mathrm{topo}}$) and incidence distortion ($\mathcal{L}_{\mathrm{hyper}}$) remain completely suppressed at near-zero levels throughout the entire bifurcation interval. Consequently, the structural entropies (Persistence Entropy and the proposed $\mathrm{HE}_{\mathrm{sym}}$) do not exhibit the stark step-function discontinuities seen in Experiment 1; instead, they merely fluctuate stably around baseline noise levels. This serves as a powerful \emph{negative control}, proving that our entropy indicators are not spuriously triggered by mere spatial displacement, but are strictly sensitive only to the designated algebraic topological dimensions.
    
    \begin{figure}
        \centering
        \includegraphics[width=0.75\linewidth]{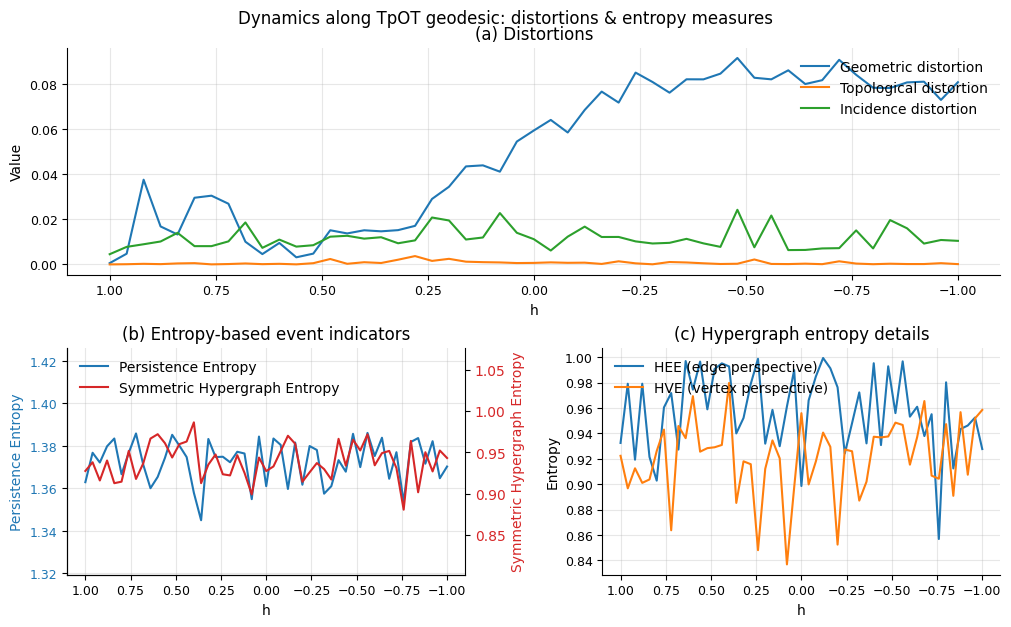}
        \caption{\revision{Ground truth baseline dynamics over parameter $h$. Because the system only undergoes a 0-dimensional transition, the 1-dimensional structural metrics ($\mathcal{L}_{\mathrm{topo}}$, $\mathcal{L}_{\mathrm{hyper}}$, and Entropies) remain completely suppressed, while the geometric distortion $\mathcal{L}_{\mathrm{geom}}$ rises to capture the spatial convergence.}}
        \label{fig:double_well_curves_gt}
    \end{figure}
    
    \paragraph{Experimental Validation of Dimension-Selective Tracking.} 
    We then subsampled this trajectory into four highly sparse keyframes and applied our reconstruction-based TpOT interpolation (Algorithm \ref{alg:overall}) over $\tau \in [0,1]$. 
    
    The interpolated trajectory captures both the spatial dynamics and the dimensional decoupling. As demonstrated in Figure \ref{fig:double_well_curves}(a), the dynamic curves computed along the geodesic $\tau$ closely aligning with the ground truth: the interpolated geometric curve $\mathcal{L}_{\mathrm{geom}}$ rises to capture the spatial fusion, while the $H_1$-specific topological components and entropy indicators remain correctly invariant. This confirms that our interpolation framework reliably reconstructs both the presence and the \emph{absence} of topological phenomena, preserving the strict decoupling between micro-scale geometric diffusion and macro-scale homological persistence across arbitrarily sparse temporal observations.

    \begin{figure}[H]
        \centering
           \begin{subfigure}[b]{0.8\textwidth}
            \centering
            \includegraphics[width=\textwidth]{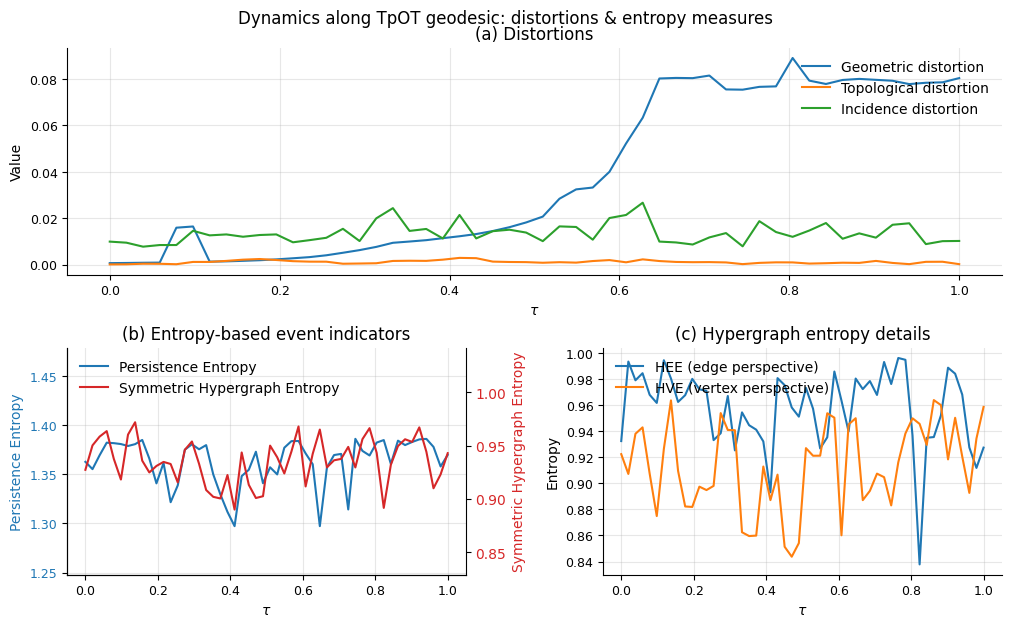}
            \caption{\revision{Reconstructed structural dynamics}}
            \label{fig:ex2_a}
        \end{subfigure}
        \hfill 
        \begin{subfigure}[b]{0.8\textwidth}
            \centering
            \includegraphics[width=\textwidth]{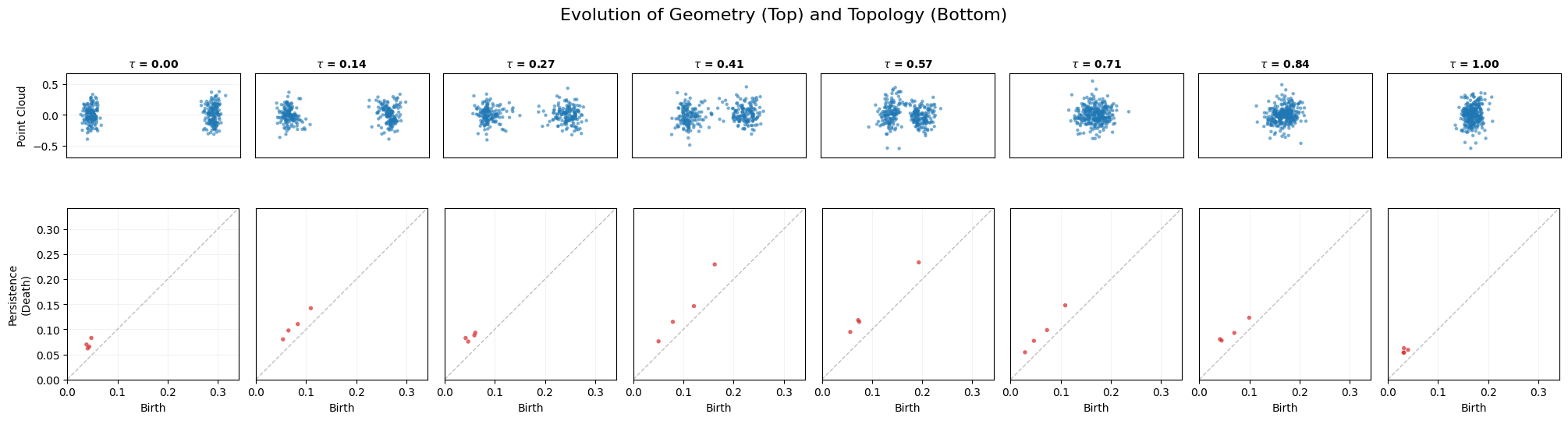}
            \caption{\revision{Reconstructed spatial geometry and corresponding persistence diagrams}}
            \label{fig:ex2_b}
        \end{subfigure}
        \caption{\revision{\textbf{Dimension-specificity and negative control evaluation.} (a)Dynamic reconstruction via geodesic interpolation from only four keyframes successfully reproduces this exact dimension-specific decoupling. (b)The MDS-based isometric embedding and corresponding 1-D persistence diagram.}}
        \label{fig:double_well_curves}
    \end{figure}
    }

    \revision{
    \subsection{Experiment 3: Self-Organization and Dimensionality in Biological Aggregations}
    
    \paragraph{Data Generation and Physical Model.} 
    To demonstrate the capability of our framework in analyzing higher-dimensional, real-world biological phenomena, we turn to the well-known D'Orsogna model of biological aggregations (e.g., bird flocks and fish schools)\cite{chuang2007state,d2006self,levine2000self,topaz2015topological}. We utilized the publicly available simulation dataset from Topaz's study, which tracks the complex topological self-organization of $N=500$ self-propelled interacting particles.
    
    Unlike the previous purely spatial models, the state of this system must be described in a 4-dimensional phase space $(x, y, v_x, v_y)$, as the particles' orientations and velocities intrinsically dictate the collective topological state (e.g., distinguishing a single mill from a double mill). The agents obey Newtonian dynamics driven by self-propulsion, friction, and a pairwise attractive-repulsive interaction potential. 
    Over time, the system spontaneously self-organizes from a relatively disorganized, disk-like swarm into a highly structured ``mill'' or vortex state, characterized by particles rotating around a hollow core. In the 4-dimensional phase space, the emergence of this hollow core and the rotational flow corresponds to the birth of prominent 1-dimensional homological features ($H_1$ topological circles). We extracted a uniformly spaced sequence of 61 snapshots from $T=1.00$ to $T=60.00$ to serve as our empirical sequence.
    }
    \revision{
    
    \begin{figure}[H]
        \centering
        \includegraphics[width=\linewidth]{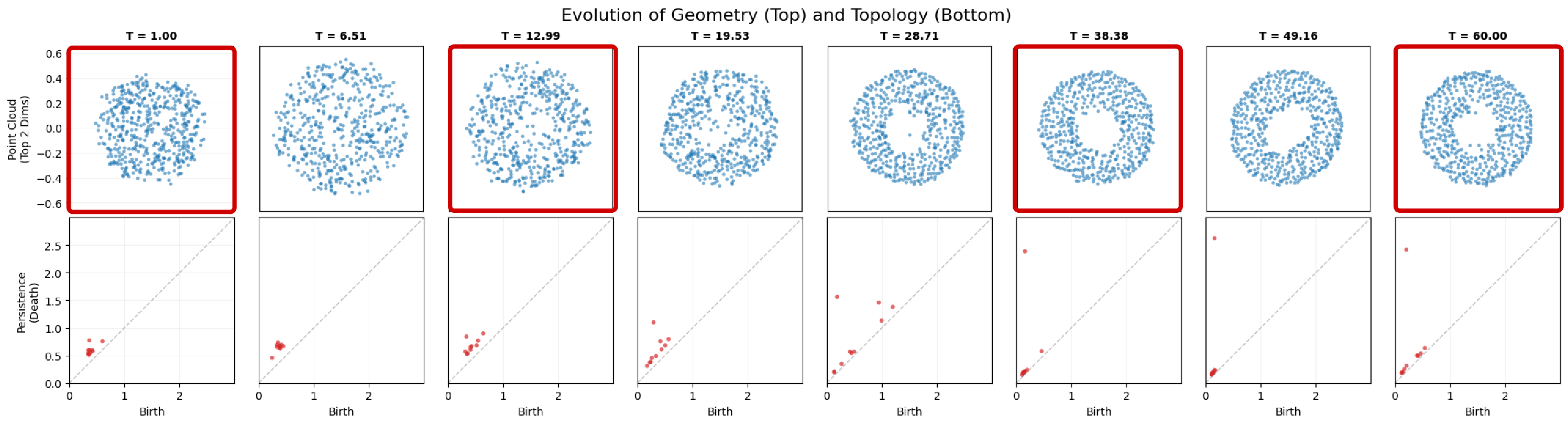}
        \caption{\revision{Ground truth evolution of the D'Orsogna biological aggregation model. (Top) The 2-dimensional spatial projection $(x,y)$ of the 4-dimensional phase space shows particles self-organizing from a disorganized state into a mill with a hollow core. \textbf{Four sparse snapshots} are chosen as our training density samples for topological optimal transport task (illustrated in \textbf{red boxes}). (Bottom) The corresponding 1-dimensional persistence diagrams track the emergence of a robust Betti-1 feature representing the rotational vortex.}}
        \label{fig:dorsogna_gt_scatter}
    \end{figure}
    
    \paragraph{Baseline Evaluation: Self-Organization and Entropy Drop.} 
    We first evaluate the baseline dynamics computed directly on the full 61-frame sequence, using the initial unstructured state at $T=1.00$ as the global reference $P^{(1)}$. As visualized in Figure \ref{fig:dorsogna_gt_scatter}, the system gradually develops a prominent persistent loop in $H_1$.
    
    Quantitatively, the distortion curves in Figure \ref{fig:dorsogna_curves}(a) effectively capture the multi-scale hierarchical relaxation of this self-organization. At the macro-scale, the topological distortion ($\mathcal{L}_{\mathrm{topo}}$) rises substantially as the initial trivial topology develops a robust Betti-1 vortex. At the meso- and micro-scales, the incidence ($\mathcal{L}_{\mathrm{hyper}}$) and geometric ($\mathcal{L}_{\mathrm{geom}}$) distortions track the continuous physical rearrangement of particles entering the annular flow.
    }
    \revision{
    The structural entropies correlate with the system’s dynamical shifts, suggesting their utility as \emph{early-warning indicators} for self-organization. As shown in Figure \ref{fig:dorsogna_curves}(b) and (c), unlike the topological distortion ($\mathcal{L}_{\mathrm{topo}}$) which exhibits a continuous, gradual rise as the physical optimal transport costs accumulate, the Symmetric Hypergraph Entropy experience a sharp, discontinuous drop from a highly entropic state ($\approx 1.0$) to a low-entropy state ($\approx 0.7$). This discrete step-function behavior clearly demonstrates a rapid state-to-state phase transition.

    Furthermore, a close examination of the temporal timeline reveals a distinct anticipatory sensitivity: the proposed Symmetric Hypergraph Entropy ($\mathrm{HE}_{\mathrm{sym}}$) triggers the abrupt entropy drop visibly earlier than both the topological distortion and the Persistence Entropy ($\mathrm{PE}$). This temporal precedence is theoretically well-founded within our macro-meso-micro framework. While $\mathcal{L}_{\mathrm{topo}}$ and $\mathrm{PE}$ require the global macroscopic loop ($H_1$) to fully mature and achieve a measurable barcode lifespan, $\mathrm{HE}_{\mathrm{sym}}$ continuously evaluates the instantaneous \emph{mesoscopic} incidence distribution. Consequently, it successfully detects the initial localized structural reorganizing---as agents just begin to align into the annular flow---well before the global vortex fully forms. 

    This mathematical behavior reflects the physical transition from a chaotic swarm to a deterministic mill state.
    
    \begin{figure}[H]
        \centering
        \includegraphics[width=\linewidth]{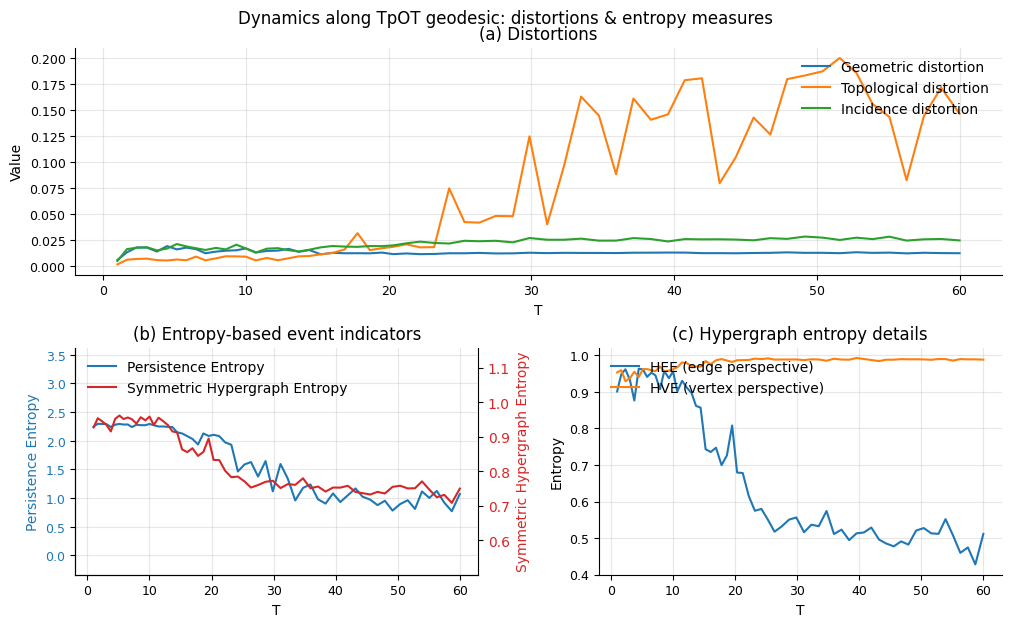}
        \caption{\revision{Baseline dynamic evaluation of the 4D D'Orsogna model. (a) The topological distortion $\mathcal{L}_{\mathrm{topo}}$ rises significantly to capture the formation of the homological loop. (b, c) Unlike previous bifurcation experiments, the structural entropies undergo a significant drop, rigorously quantifying the thermodynamic transition from a disorganized chaotic swarm into a highly self-organized, low-entropy vortex state.}}
        \label{fig:dorsogna_curves}
    \end{figure}
    
    \textbf{Dynamic Reconstruction of Topological Phase Transitions.} 
    To test our framework's capacity to handle high-dimensional phase spaces with extreme temporal sparsity, we aggressively subsampled the 61-frame sequence into merely four keyframes. We then applied Algorithm \ref{alg:overall} to reconstruct the 4-dimensional topological geodesic parameterized by $\tau \in [0,1]$. 
    }
    \revision{
    
    \begin{figure}[H]
        \centering
        \includegraphics[width=\linewidth]{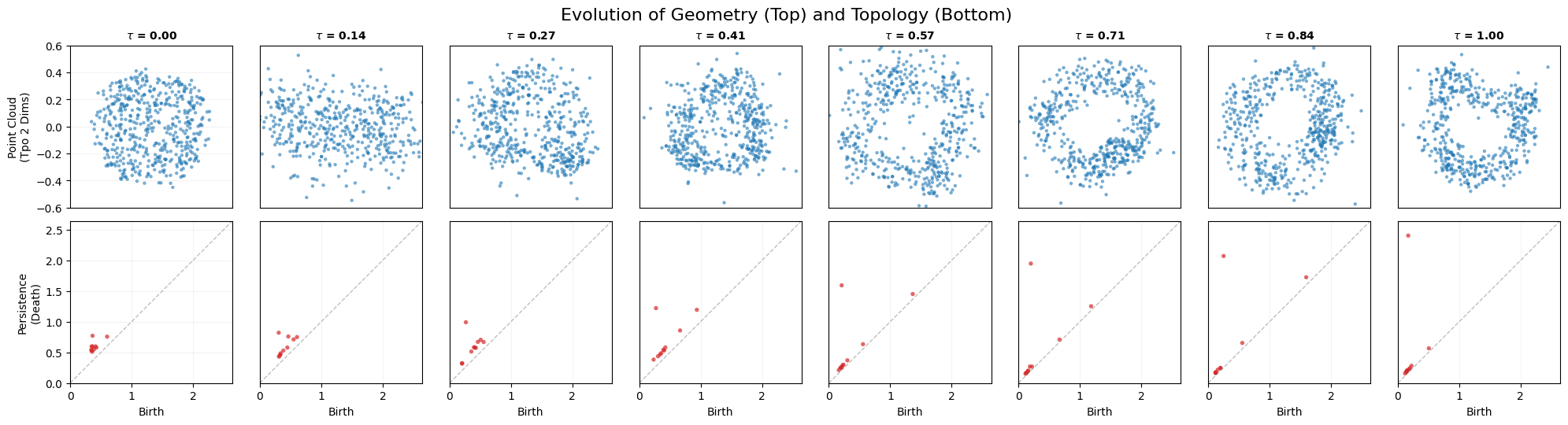}
        \caption{\revision{Reconstructed spatial geometry and persistence diagrams from only four keyframes. While the 2D visual projection exhibits minor alignment artifacts due to the rigid 4-dimensional optimal transport coupling, the 1-dimensional persistent homology accurately reconstructs the birth of the vortex core.}}
        \label{fig:dorsogna_interp_scatter}
    \end{figure}

    \begin{figure}[htbp]
        \centering
        \includegraphics[width=\linewidth]{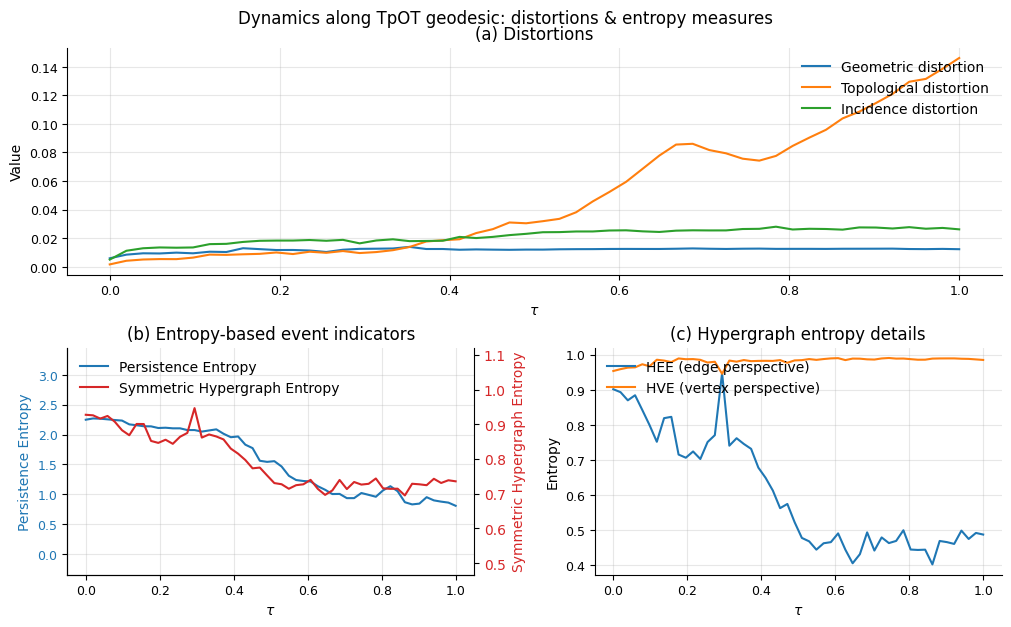}
        \caption{\revision{Experimental validation of the reconstructed structural dynamics. The TpOT geodesic interpolation tracks the ground truth dynamics from Figure \ref{fig:dorsogna_curves}, recovering both the the increase in topological distortion and the subsequent entropy drop associated with biological self-organization.}}
        \label{fig:dorsogna_interp_curves}
    \end{figure}
    
    As shown in Figure \ref{fig:dorsogna_interp_scatter}, our MDS-based geometric reconstruction interpolates the 4-dimensional phase space. It is worth noting that while the 2-dimensional $(x,y)$ spatial projections of the interpolated states may exhibit slight visual alignment artifacts, these are purely cosmetic. This visual discrepancy is a natural consequence of projecting and rigidly aligning the full 4-dimensional $(x,y,v_x,v_y)$ optimal transport coupling into a 2D viewing plane. Thus, these alignment artifacts have no impact on the quantitative results of our framework. Because our dynamic evaluation---including both the persistent homology reconstruction and the TpOT cost computation---relies strictly on the intrinsic pairwise distance function $k_\tau$, the mathematical outcomes are invariant to rigid rotational alignments used for visualization. The underlying topological structure and structural entropies are highly preserved.

    This is definitively proven by the dynamic curves computed along the interpolated geodesic $\tau$ (Figure \ref{fig:dorsogna_interp_curves}). The reconstructed trajectories mirror the ground truth sequence: the framework exactly recovers the continuous rise in topological distortion $\mathcal{L}_{\mathrm{topo}}$ and the precise trajectory of the entropy drop. This validates that our algorithm robustly captures complex, high-dimensional biological self-organization from highly sparse observations, maintaining strict mathematical accuracy even when low-dimensional visual projections become challenging.
    }

    \subsection{Real‐World Data: Stroke fMRI}\label{ssec:real}

     \subsubsection{Data description and preprocessing}  
    Let the fMRI data at two time points be denoted as
    $\mathcal{X}^{(M)} \in \mathbb{R}^{n_{x} \times n_{y} \times n_{z} \times T_{M}}, \quad \quad \mathcal{X}^{(Y)} \in \mathbb{R}^{n_{x} \times n_{y} \times n_{z} \times T_{Y}},$
    where $(i, j, k)$ indexes voxel coordinates in the spatial domain $\Omega$, and $t$ indexes time.

    To obtain a stable voxel representation, we compute the mean of the BOLD signal:
    \[
    \bar{\mathcal{X}}^{(M)}(i,j,k)=\frac{1}{T_{M}}\sum_{t=1}^{T_{M}}\mathcal{X}^{(M)}(i,j,k,t),\quad\bar{\mathcal{X}}^{(Y)}(i,j,k)=\frac{1}{T_{Y}}\sum_{t=1}^{T_{Y}}\mathcal{X}^{(Y)}(i,j,k,t).
    \]
    
    Each scan is a 4D volume of size \(64\times64\times34\times60\). For each time point, we compute the voxel-wise temporal average, producing two 3D volumes of size \(64\times64\times34\).  

    The human cerebral cortex exhibits highly heterogeneous patterns of functional connectivity. \revision{To eliminate voxel-level noise and enhance interpretability, we utilize the widely adopted Yeo7 atlas \cite{yeo2011organization} to partition the spatial domain $\Omega$ into seven functionally coherent regions. After aligning the atlas to the subject's native space, we obtain a discrete partition $\Omega = \bigcup_{v=1}^7 \Omega_v$, where $v \in \{1, \dots, 7\}$ indexes the large-scale functional networks (e.g., Visual, Somatomotor, Default Mode).}

    Each voxel is represented by a 4-dimensional feature vector incorporating spatial coordinates and the BOLD signal:
    $$x^{(M)}_{i,j,k}=(i,j,k,\:\bar{\mathcal{X}}^{(M)}(i,j,k)\:)^\top,\quad x^{(Y)}_{i,j,k}=(i,j,k,\:\bar{\mathcal{X}}^{(Y)}(i,j,k)\:)^\top.$$

    \revision{The region-specific point clouds for each functional network $v$ are thus defined as:}
    $$P_v^{(M)}=\{x^{(M)}_{i,j,k}\mid(i,j,k)\in\Omega_v^{(M)}\},\quad P_v^{(Y)}=\{x^{(Y)}_{i,j,k}\mid(i,j,k)\in\Omega_v^{(Y)}\}.$$

    Finally, we form the full voxel-wise datasets:
    $$X_M=\bigcup\limits_{v=1}^7P_v^{(M)}\in\mathbb{R}^{N_M\times4},\quad X_Y=\bigcup\limits_{v=1}^7P_v^{(Y)}\in\mathbb{R}^{N_Y\times4},$$
    with corresponding label vectors $\ell_M\in\{1,...,7\}^{N_M},\quad\ell_Y\in\{1,...,7\}^{N_Y}.$

    \begin{figure}[htbp]
      \centering
      \begin{subfigure}[b]{0.35\textwidth}
        \centering
        \includegraphics[width=\textwidth]{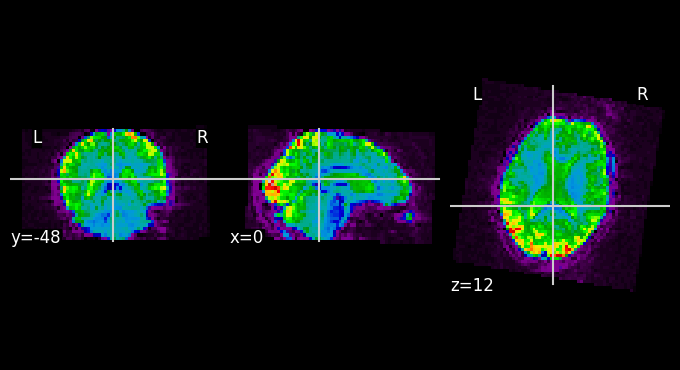}
        \caption{Raw fMRI volume (3-month)}
        \label{fig:original_fmri_data}
      \end{subfigure}
      \hfill 
      \begin{subfigure}[b]{0.6\textwidth}
        \centering
        \includegraphics[width=\textwidth]{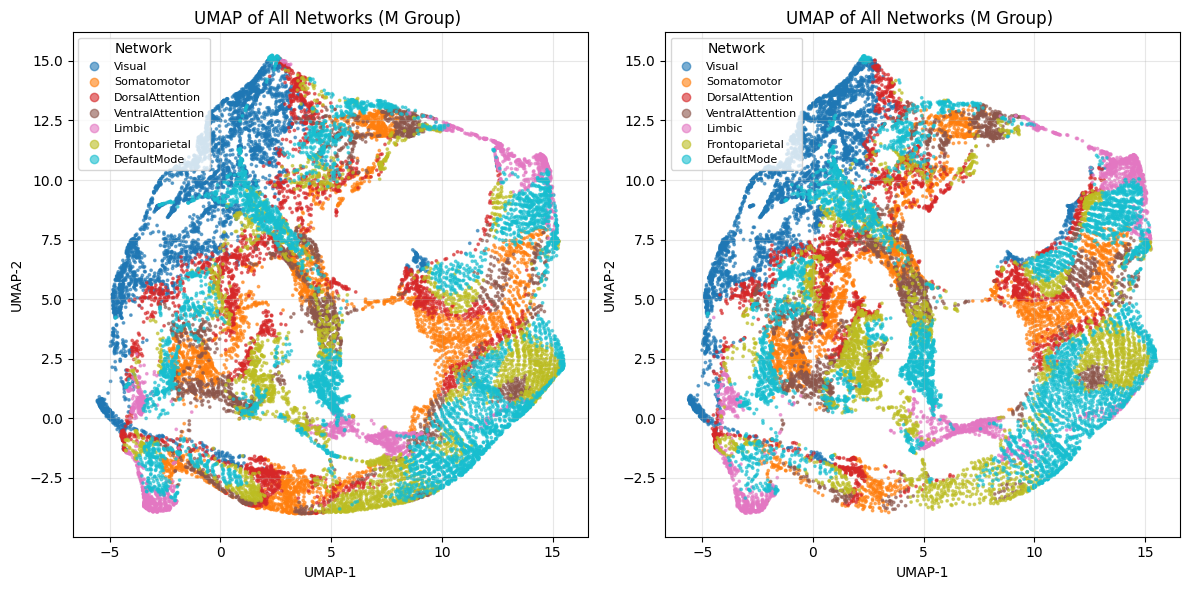}
        \caption{UMAP projections of parcellated networks}
        \label{fig:UMAP_projection_of_7_networks}
      \end{subfigure}
      \caption{\revision{Visualization of the stroke patient's fMRI data and dimensionality reduction. (a) An exemplary 3D spatial slice of the patient's raw fMRI volume. (b) 2D UMAP embeddings of the voxel feature vectors at 3-month (left) and 1-year (right) post-stroke time points. Points are colored by their corresponding Yeo7 functional network assignments.}}
      \label{fig:fmri_visualization_combined}
    \end{figure}

    \subsubsection{Embedding and network construction}  

    To visualize and analyse the stroke fMRI volumes at 3 months and 12 months, we applied the following dimensionality-reduction pipeline:

    \revision{\emph{Standardization and UMAP embedding.}  
    The voxel-wise datasets $X_M$ and $X_Y$ are standardized to have zero mean and unit variance. Uniform Manifold Approximation and Projection (UMAP) is then applied to map the standardized data from $\mathbb{R}^4$ to $\mathbb{R}^2$, yielding the respective low-dimensional embeddings $Z_M, Z_Y \subset \mathbb{R}^2$.} The result is shown in Figure \ref{fig:fmri_visualization_combined}(b).

    Within the embeddings, each Yeo7 network corresponds to a regional subset:
    $$Z_{M,v} = \{z_i \in Z_M \mid \ell_M[i] = v\}, \qquad Z_{Y,v} = \{z_i \in Z_Y \mid \ell_Y[i] = v\}.$$
    These subsets are essential for visualizing inter-network geometry and for constructing hypergraphs over the embedded regions in subsequent entropy-based analyses. 
      
    These 2D embeddings are both resampled to 600 points and then serve as the input point clouds for our TpOT analysis.  

    We then construct 1D persistent homology (retaining the top 20 persistence pairs to represent the dominant topological features), binary incidence matrices, and measure topological networks.
    
\revision{
\subsubsection{Multi-Scale Topological Analysis and Findings}  
The entropic TpOT problem is solved as before, and distortions are computed on an interpolation grid of 51 points spanning the geodesic between the 3-month and 12-month post-stroke states. 

To characterize the dynamic reconfiguration of functional brain organization, we evaluated the embedded voxel sets through a hierarchical set of indicators: the macroscopic metrics (Topological Distortion $\mathcal{L}_{\mathrm{topo}}$ and Persistence Entropy $\mathrm{PE}$), our proposed mesoscopic dual-perspective framework ($\mathrm{HE}_V$, $\mathrm{HE}_E$, and the aggregated $\mathrm{HE}_{\mathrm{sym}}$), and the baseline Geometric Distortion ($\mathcal{L}_{\mathrm{geom}}$). 

Figure \ref{fig:six_heatmaps} displays the heatmaps of these six indicators across the seven Yeo brain areas (vertical axis) and the interpolated temporal parameter $\tau \in [0,1]$ (horizontal axis). A comparative analysis across these heatmaps illustrates the multi-scale structural dynamics of our framework.

\paragraph{Macroscopic and Mesoscopic Dynamics.} 
The macroscopic indicators, $\mathrm{PE}$ and $\mathcal{L}_{\mathrm{topo}}$ (Figures \ref{fig:six_heatmaps}d, e), follow continuous trajectories and show similar global trends. Both metrics capture the distinct topological evolution (sharply rise at $\tau \approx 0.4$) in Network 2 (Somatomotor), characterizing structural shifts at a macroscopic scale. 

At the mesoscopic scale, the three hypergraph entropies (Figures \ref{fig:six_heatmaps}a-c) are more sensitive to local and transient topological reorganizations. These metrics reveal region-specific reorganization patterns that are typically obscured by macroscopic measures. Unlike synthetic datasets where transitions are often uniform, the real-world fMRI data shows a clear \emph{decoupling} between $\mathrm{HE}_V$ and $\mathrm{HE}_E$:
\begin{itemize}
    \item In certain brain networks, such as Network 2 (Somatomotor) and Network 3 (Dorsal Attention), $\mathrm{HE}_V$ and $\mathrm{HE}_E$ display opposite temporal trends.
    \item In Network 4 (Ventral Attention), $\mathrm{HE}_V$ remains stable while $\mathrm{HE}_E$ decreases around $\tau \approx 0.9$. 
    \item Conversely, in others such as Network 7 (Default Mode), $\mathrm{HE}_V$ and $\mathrm{HE}_E$ follow synchronized trajectories. 
\end{itemize}
This asynchronous behavior suggests that brain functional reorganization is complex, with node participation ($\mathrm{HE}_V$) and functional loop uniformity ($\mathrm{HE}_E$) evolving independently during stroke recovery. Symmetric Hypergraph Entropy ($\mathrm{HE}_{\mathrm{sym}}$) integrates these variations. As shown in Figure \ref{fig:six_heatmaps}a, $\mathrm{HE}_{\mathrm{sym}}$ combines features from both vertex- and edge-perspective entropies to represent mesoscopic structural transitions.

\paragraph{Microscopic Geometric Distortion.}
It is worth noting that the Geometric Distortion ($\mathcal{L}_{\mathrm{geom}}$, Figure \ref{fig:six_heatmaps}f) displays an almost uniform linear growth across all brain regions. This indicates that while pure geometric optimal transport distances reliably track the overall spatial displacement of the point clouds, they are naturally less sensitive to the complex topological phase transitions occurring within the functional networks during the recovery process.

\paragraph{Vertex-Level Localization.}
To estimate the spatial distribution of these structural transitions, Figure \ref{fig:seg7_geodesic} illustrates the vertex-level hypergraph entropy field on the Dorsal Attention Network. We mapped the absolute cycle-level entropy variation $|\Delta \mathrm{HE}| = |\mathrm{HE}_\tau - \mathrm{HE}_0|$ back to the vertex domain (as defined in Eq. \ref{eq:point_score}). Each point represents a vertex in the resampled UMAP embedding, and its color encodes the propagated entropy change. Regions highlighted in red correspond to vertices that participate in cycles exhibiting the strongest entropy variations, indicating localized structural reorganization within the hypergraph. This spatial map, complemented by the exact embeddings at 3 months and 1 year (Figures \ref{fig:yeo7_M} and \ref{fig:yeo7_Y}) , demonstrates how our cycle-to-point entropy propagation effectively identifies key local transformations. Additionally, the specific distortion curves on the Dorsal Attention Network is presented in Figure \ref{fig:yeo7_curve}
}

\begin{figure}[htbp]
  \centering
  \begin{subfigure}[t]{0.32\textwidth}
    \centering
    \includegraphics[width=\textwidth]{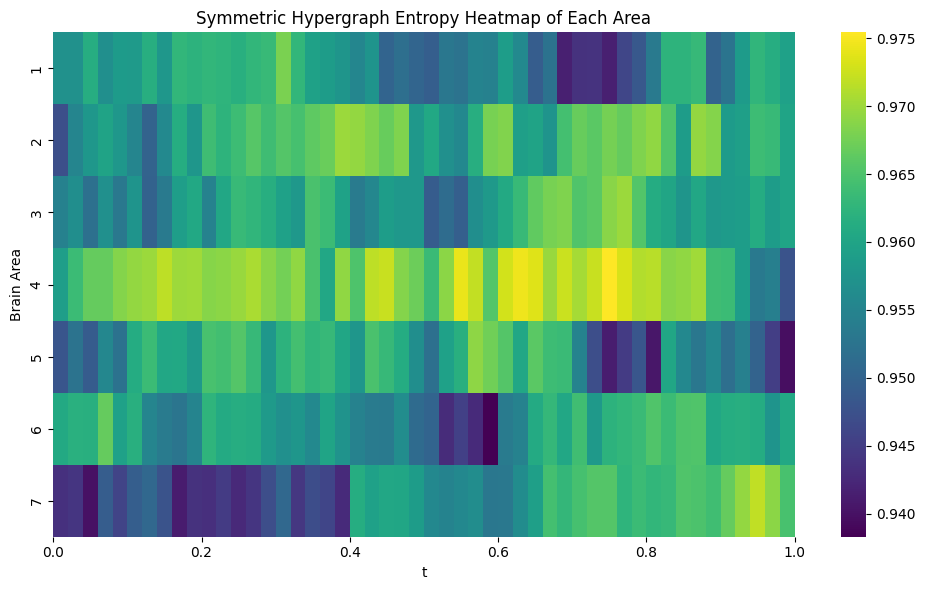} 
    \caption{Symmetric Entropy ($\mathrm{HE}_{\mathrm{sym}}$)}
    \label{fig:heatmap_sym}
  \end{subfigure}
  \hfill 
  \begin{subfigure}[t]{0.32\textwidth}
    \centering
    \includegraphics[width=\textwidth]{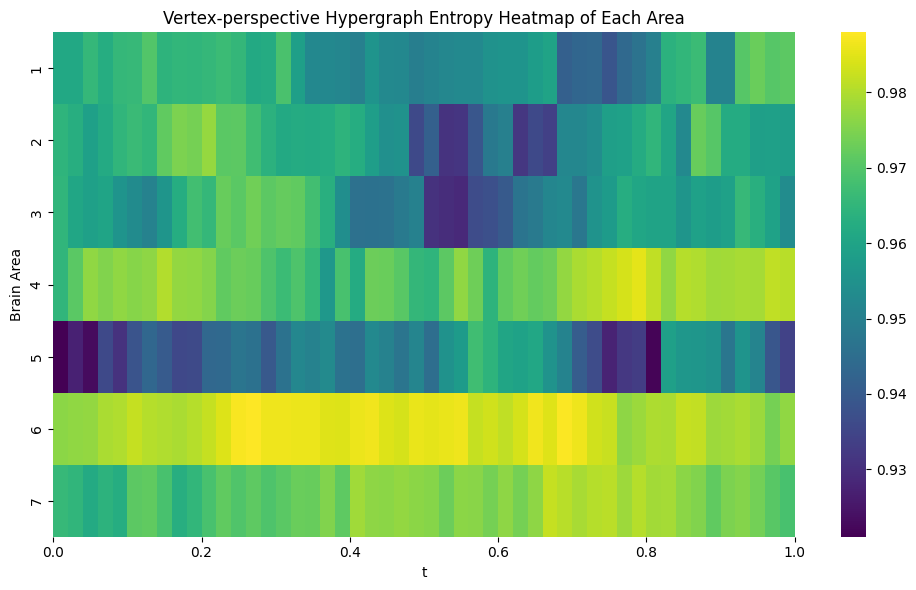} 
    \caption{Vertex-perspective Entropy ($\mathrm{HE}_V$)}
    \label{fig:heatmap_V}
  \end{subfigure}
  \hfill 
  \begin{subfigure}[t]{0.32\textwidth}
    \centering
    \includegraphics[width=\textwidth]{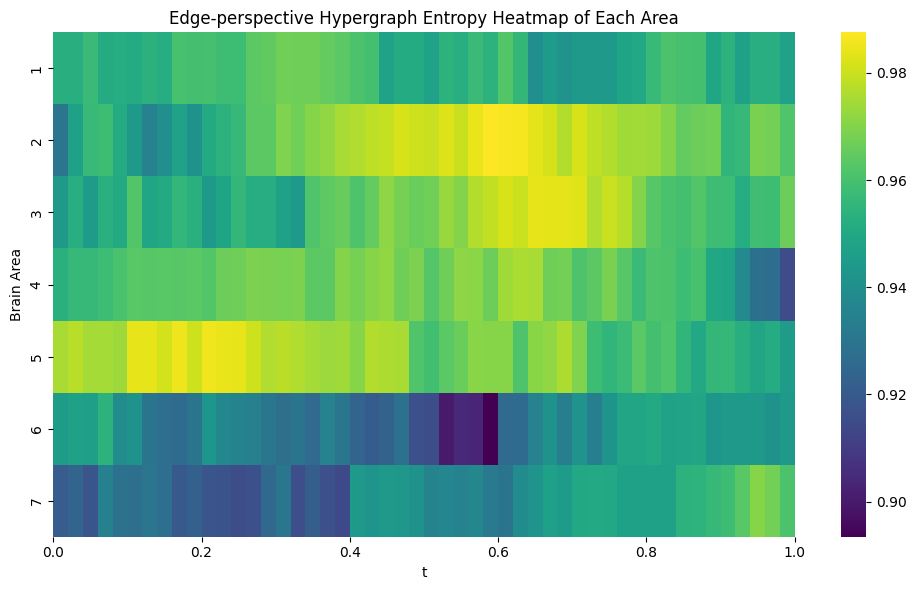} 
    \caption{Edge-perspective Entropy ($\mathrm{HE}_E$)}
    \label{fig:heatmap_E}
  \end{subfigure}
  
  \vspace{0.3cm} 
  
  \begin{subfigure}[t]{0.32\textwidth}
    \centering
    \includegraphics[width=\textwidth]{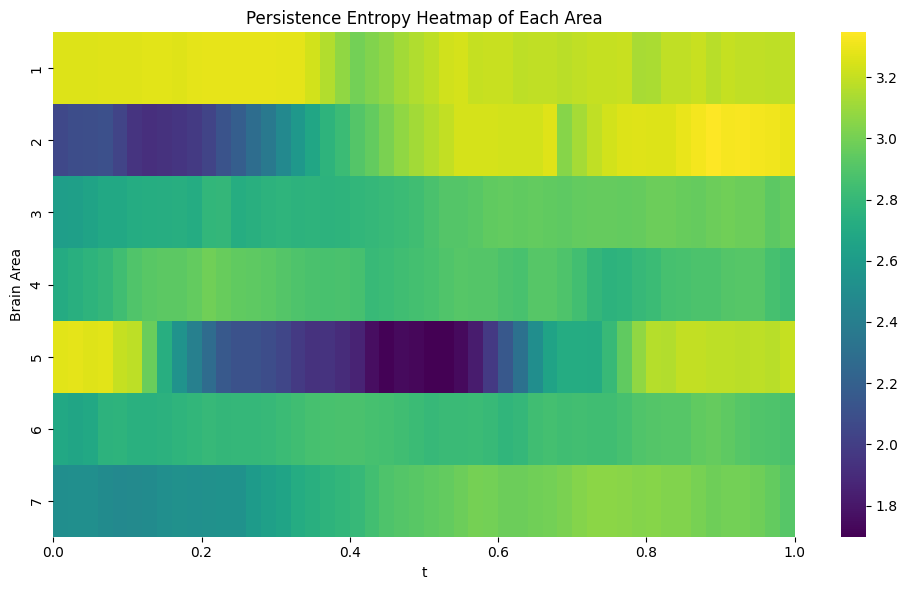} 
    \caption{Persistence Entropy ($\mathrm{PE}$)}
    \label{fig:heatmap_PE}
  \end{subfigure}
  \hfill 
  \begin{subfigure}[t]{0.32\textwidth}
    \centering
    \includegraphics[width=\textwidth]{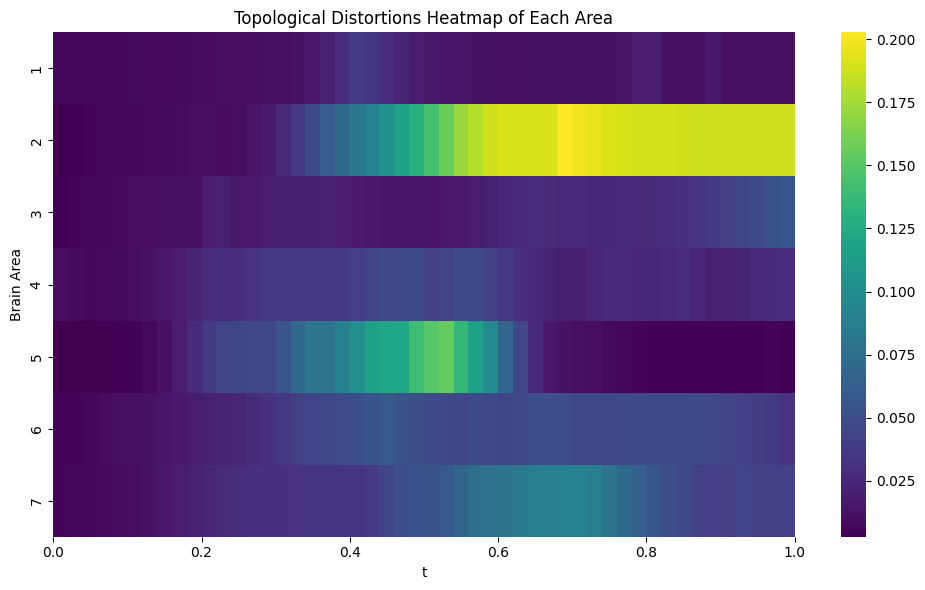} 
    \caption{Topological Distortion ($\mathcal{L}_{\mathrm{topo}}$)}
    \label{fig:heatmap_topo}
  \end{subfigure}
  \hfill 
  \begin{subfigure}[t]{0.32\textwidth}
    \centering
    \includegraphics[width=\textwidth]{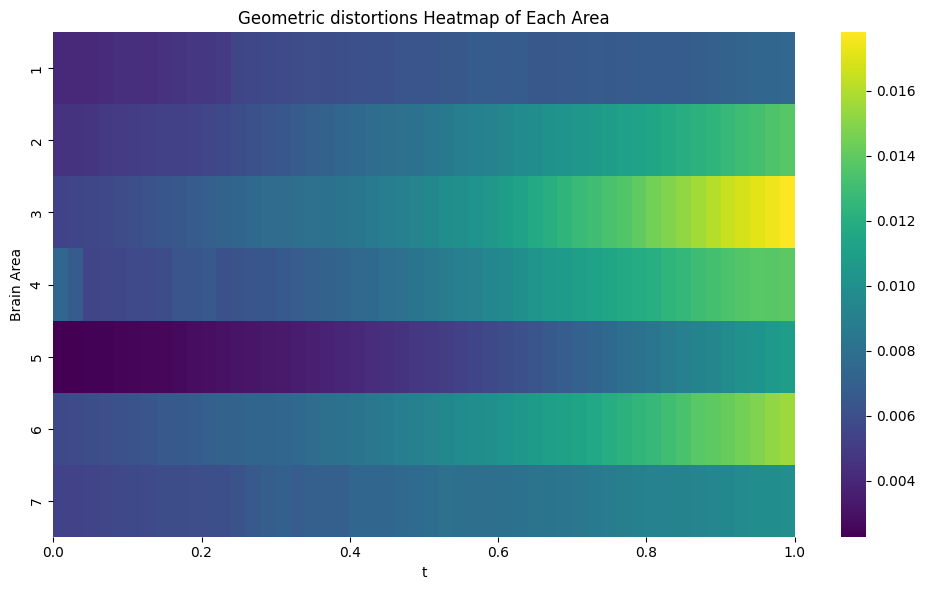} 
    \caption{Geometric Distortion ($\mathcal{L}_{\mathrm{geom}}$)}
    \label{fig:heatmap_geo}
  \end{subfigure}
  \caption{\revision{\textbf{Dynamic entropy and distortion heatmaps across seven functional brain networks (Yeo7) along the TpOT geodesic.} The mesoscopic hypergraph entropies (a-c) capture asynchronous local rewirings and region-specific variations. In contrast, the macroscopic Persistence Entropy (d) and Topological Distortion (e) characterize global structural shifts (e.g., in Network 2). The Geometric Distortion (f) exhibits uniform linear growth, tracking overall spatial displacement rather than topological phase transitions.}}
  \label{fig:six_heatmaps}
\end{figure}

    \begin{figure}[htbp]
      \centering
      \begin{subfigure}[t]{0.33\textwidth}
          \centering
          \includegraphics[width=\textwidth]{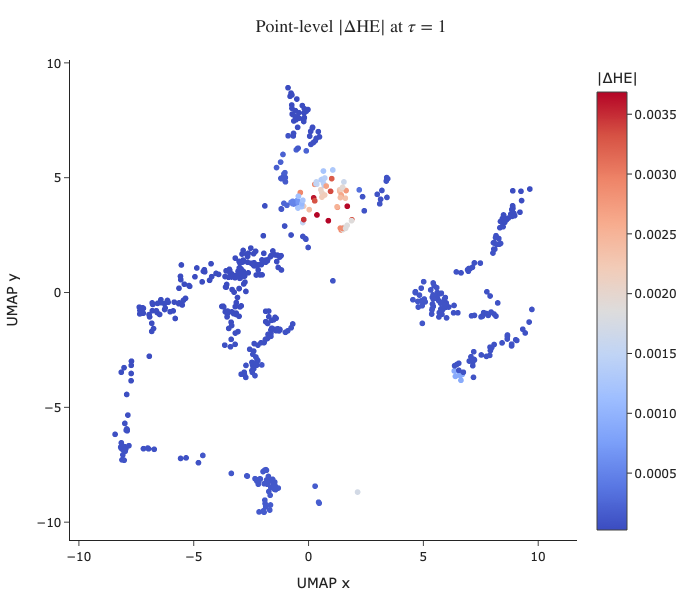}
          \caption{Visualization of vertex-level hypergraph entropy change in Dorsal Attention Network defined as (\ref{eq:point_score})}
          \label{fig:seg7_geodesic}
      \end{subfigure}
      \centering
      \begin{subfigure}[t]{0.32\textwidth}
        \centering
        \includegraphics[width=\textwidth]{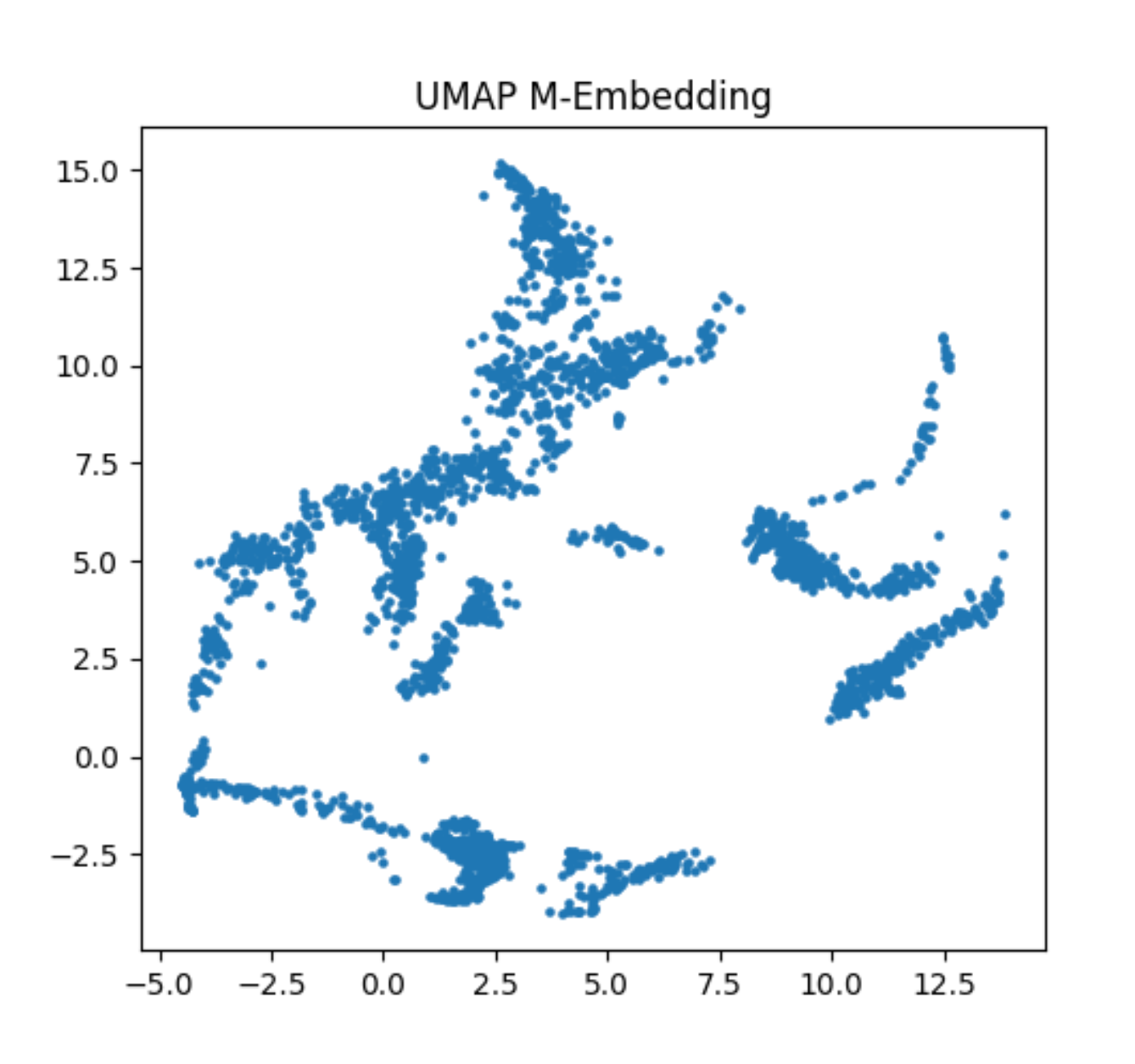}
        \caption{Embeddings of the Dorsal Attention network at 3 month}
        \label{fig:yeo7_M}
      \end{subfigure}
    \hfill
      \begin{subfigure}[t]{0.32\textwidth}
        \centering
        \includegraphics[width=\textwidth]{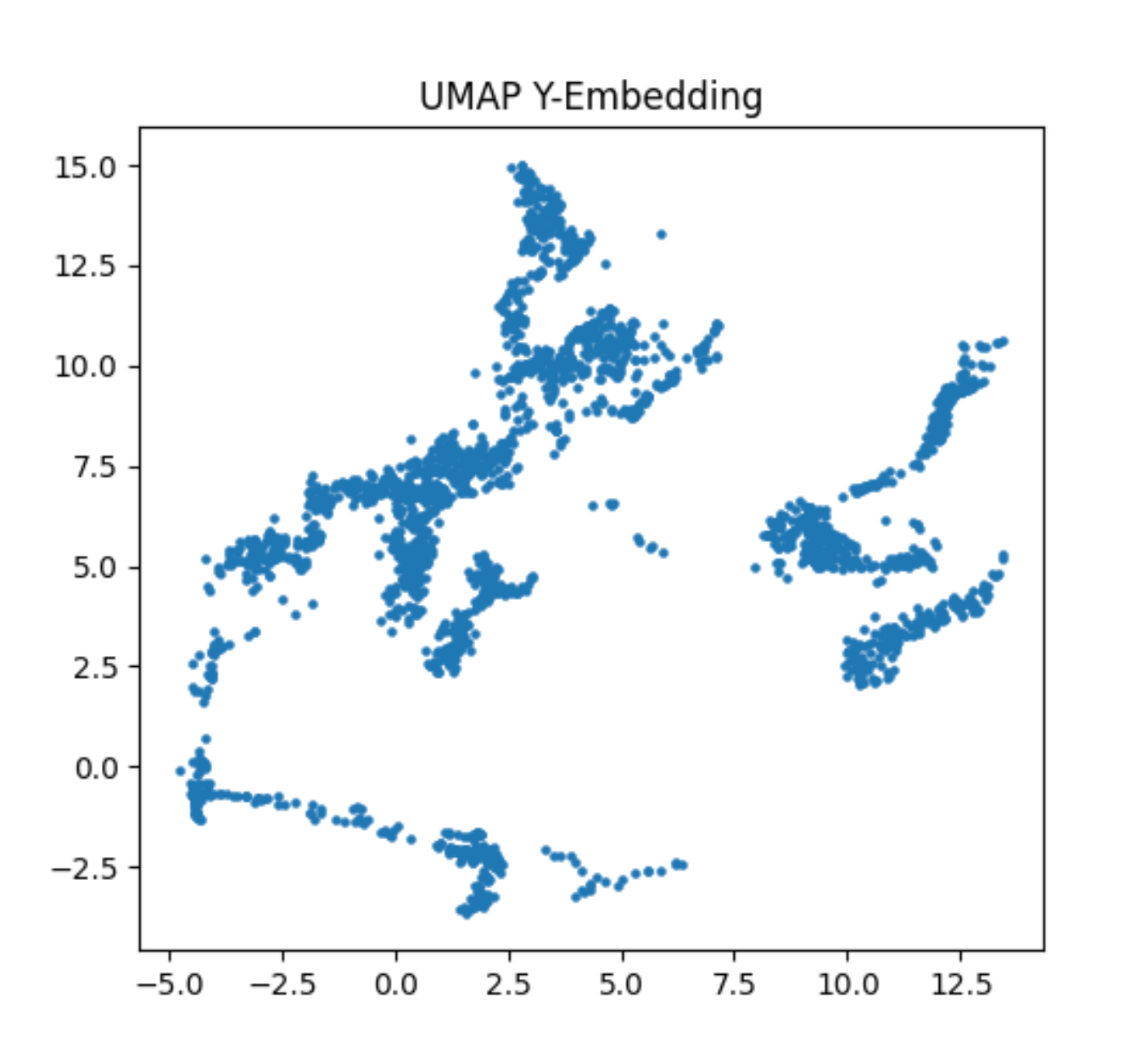}
        \caption{Embeddings of the Dorsal Attention network at 1 year}
        \label{fig:yeo7_Y}
      \end{subfigure}
      \caption{Vertex-level hypergraph entropy analysis of the Dorsal Attention network.}
      \label{fig:yeo7_MY}
    \end{figure}

    \begin{figure}[htbp]
        \centering
        \includegraphics[width=0.8\linewidth]{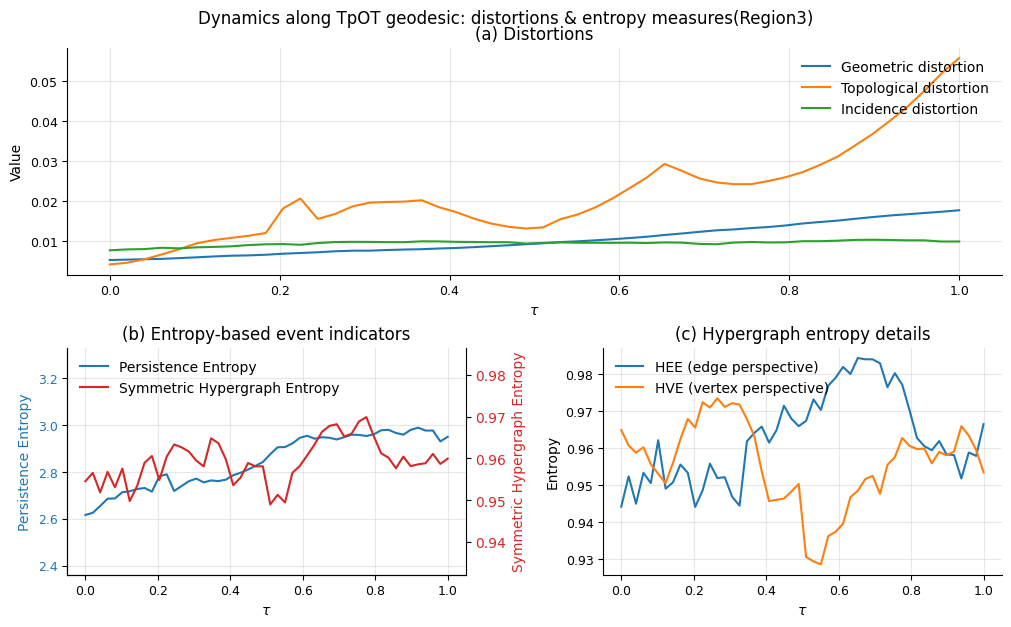}
        \caption{Distortion and entropy trajectories along the TpOT geodesic of the Dorsal Attention network.}
        \label{fig:yeo7_curve}
    \end{figure}

\paragraph{Computational details.}  
All experiments were implemented in Python and using Ripserer.jl for persistent homology.  Synthetic runs required approximately 20 seconds per \revision{frame} on an M3 Pro CPU; the real‐world experiment took approximately 300 seconds.

\revision{
\paragraph{Conclusion.} 
Experiments on stochastic models and biological systems show that the TpOT framework recovers topological phase transitions from sparse temporal observations. The dynamic distortion components ($\mathcal{L}_{\mathrm{geom}}$, $\mathcal{L}_{\mathrm{hyper}}$, $\mathcal{L}_{\mathrm{topo}}$) distinguish continuous physical deformations from discrete structural jumps, identifying a hierarchical critical transition across macro-, meso-, and micro-scales.

Our structural entropy indicators serve as threshold-free markers for various topological events. While Persistence Entropy tracks the lifespan of global homological features, Symmetric Hypergraph Entropy ($\mathrm{HE}_{\mathrm{sym}}$) acts as an early-warning indicator for both bifurcations (indicated by entropy jumps) and biological self-organization (indicated by entropy drops). Furthermore, negative control evaluations confirm the dimension-selectivity and specificity of these metrics.

Finally, application to stroke fMRI data illustrates the utility of the dual-perspective framework. The decoupling between vertex- and edge-perspective entropies reveals asymmetric cortical reorganization. Combined with the point-level hypergraph-entropy field to localize the structural origins of these transitions, our method provides a multi-scale approach for analyzing dynamic topological phenomena in complex systems.
}

\section{Summary and Future Work}\label{sec:conclusion}
In this paper, we introduced a rigorous mathematical framework for tracking dynamic structural transitions in time-varying point clouds. \revision{By utilizing a topological and hypergraph reconstruction strategy instead of direct abstract network interpolation, our method yields continuous topological trajectories} from sparse temporal snapshots. Along these trajectories, we proposed a \revision{hierarchical} evaluation framework. We demonstrated that \revision{macroscopic metrics (such as PE and Topological Distortion) are well-fitted for capturing global evolutions, whereas our proposed mesoscopic dual-perspective Hypergraph Entropy ($\mathrm{HE}_V$ and $\mathrm{HE}_E$) provides a highly sensitive lens for detecting transient, asynchronous local rewirings}. Our validation across \revision{physical, biological, and neuroimaging datasets} confirms the specificity and complementarity of these \revision{multi-scale} indicators.

\revision{The real-world stroke fMRI experiment in this study served primarily as a methodological proof-of-concept to demonstrate the computational sensitivity of our mathematical framework.} Looking ahead, our future work will focus on \revision{extending this framework to large-scale longitudinal clinical cohorts, statistically correlating our dynamic hypergraph entropy curves with cognitive and motor recovery scores to establish robust topological biomarkers for post-stroke rehabilitation. Our future work will explore the underlying biological mechanisms and information geometry driving the entropy decoupling observed in these functional brain networks. We aim to mathematically explain how geometric deformations of probability support sets drive this asynchronous structural evolution. Finally, developing scalable approximations for TpOT and topological extraction on massive, unparcellated graphs remains a crucial computational direction.}

\section*{Acknowledgements}
We would like to thank Professor Xiaosong Yang for the helpful discussions. This work was supported by the National Natural Science Foundation of China (12401233), NSFC International Creative Research Team (W2541005), National Key Research and Development Program of China (2021ZD0201300), Guangdong-Dongguan Joint Research Fund (2023A1515140016), Guangdong Provincial Key Laboratory of Mathematical and Neural Dynamical Systems (2024B1212010004), Guangdong Major Project of Basic Research (2025B0303000003), and Hubei Key Laboratory of Engineering Modeling and Scientific Computing.

\appendix
\section{Supplementary material}
    \subsection{Persistent Homology}
\label{Supplementary material:TpOT}

    Given a finite point cloud $X\subset\mathbb{R}^d$, Persistent Homology(PH) constructs a nested sequence of simplicial complexes (e.g., the Vietoris–Rips or Čech complexes) parameterized by a scale parameter $\varepsilon$\cite{zomorodian2004computing}.  As $\varepsilon$ increases, simplices are added whenever all pairwise distances among their vertices fall below $\varepsilon$, yielding a filtration
    \begin{align*}
        K_\varepsilon(X)\;:\quad K_{\varepsilon_0}\;\hookrightarrow\;K_{\varepsilon_1}\;\hookrightarrow\;\cdots\;\hookrightarrow\;K_{\varepsilon_M},
    \end{align*}
    where $ K_{\varepsilon_i}\subseteq K_{\varepsilon_{i+1}} $.  By tracking the appearance (“birth”) and disappearance (“death”) of homology classes (connected components, loops, etc.) throughout this filtration, one obtains a persistence diagram—a multiset of points $\{(b_i,d_i)\}$ in the plane, each recording the interval $(b_i,d_i)$ over which a topological feature exists.

    The multiset of lifespans $\{d_i - b_i\}$ serves as a succinct “signature” of the data’s topology: longer intervals correspond to more prominent features, while short-lived intervals often reflect topological noise.  Persistence diagrams are stable under perturbations of the input, and admit well-studied metrics such as the bottleneck and Wasserstein distances\cite{cohen2005stability,lacombe2018large}.

    \subsection{Gromov-Wasserstein and Co-Optimal Transport Distances}
    In this subsection we review three fundamental optimal-transport-based metrics that form the building blocks of the Topological Optimal Transport framework.

    Let $(X,d)$ be a Polish metric space and $\mu,\mu^\prime$ two Borel probability measures supported on X.  For $p\in1,\infty)$, the p\nobreakdash-Wasserstein distance is defined by
    \begin{align*}
        d_{W,p}(\mu,\mu^\prime)
        =\Bigl(\inf_{\pi\in\Pi(\mu,\mu^\prime)}\int_{X\times X} d(x,x^\prime)^p\mathrm{d}\pi(x,x^\prime)\Bigr)^{1/p},
    \end{align*}
    where $\Pi(\mu,\mu{^\prime})$ denotes the set of couplings (joint measures) with marginals $\mu$ and $\mu{^\prime}$, and an optimal coupling $\pi$ realises this infimum.

    When comparing persistence diagrams $D$ and $D{^\prime}$, one augments the plane with a “diagonal” point $\partial\Lambda$ to allow unmatched features, and replaces $\Pi(\mu,\mu{^\prime})$ by the set of admissible partial matchings $\Pi(D,D{^\prime})$.  The resulting diagram-Wasserstein distance is
    \begin{align*}
        d_{W,p}^{\mathrm{PD}}(D,D^\prime)^p
        =\min_{\pi\in\Pi(D,D^\prime)}\Bigl(
        \sum_{(a,b)\in\pi}\left \|a-b\right \|_p^p
        +\sum_{s\in U_\pi}\left \|s-\mathsf{Proj}{\partial\Lambda}(s)\right \|_p^p
        \Bigr),
    \end{align*}
    where $U_\pi$ are unmatched points and $\mathsf{Proj}_{\partial\Lambda}$ projects onto the diagonal\cite{lacombe2018large}.

    When the two measures live on different metric spaces $(X,d,\mu)$ and $(X{^\prime},d{^\prime},\mu{^\prime})$, the Gromov–Wasserstein (GW) distance aligns their relational structures by minimizing differences of pairwise distances.  For a coupling $\pi\in\Pi(\mu,\mu{^\prime})$, the p\nobreakdash-distortion
    \begin{align*}
        \mathrm{dis}_{\mathrm{GW},p}(\pi)
        =\Bigl(\!\iint_{(X\times X{^\prime})^2}\bigl|\,d(x,y)-d{^\prime}(x{^\prime},y{^\prime})\bigr|^p\,\mathrm{d}\pi(x,x{^\prime})\,\mathrm{d}\pi(y,y{^\prime})\Bigr)^{1/p}.
    \end{align*}

    The GW distance is then
    \begin{align*}
        d_{\mathrm{GW},p}\bigl((X,d,\mu),(X^\prime,d^\prime,\mu^\prime)\bigr)
        =\inf_{\pi\in\Pi(\mu,\mu’)}\mathrm{dis}_{\mathrm{GW},p}(\pi),
    \end{align*}
    which defines a pseudo-metric on the metric-measure spaces\cite{memoli2007use}.

    To compare two measure hypernetworks $H=(X,\mu,Y,\nu,\omega)$ and $H^\prime=(X^\prime,\mu^\prime,Y^\prime,\nu^\prime,\omega^\prime)$, where $\omega$ encodes vertex–hyperedge incidences, one seeks couplings $\pi^v\in\Pi(\mu,\mu\prime) $(vertices) and $\pi^e\in\Pi(\nu,\nu\prime)$ (hyperedges) that minimise
    \begin{align*}
        \mathrm{dis}_{\mathrm{COOT},p}(\pi^v,\pi^e)
        =\Bigl(\!\int_{X\times X^\prime\times Y\times Y^\prime}\!\bigl|\omega(x,y)-\omega^\prime(x^\prime,y^\prime)\bigr|^p
        \,\mathrm{d}\pi^v(x,x^\prime)\,\mathrm{d}\pi^e(y,y^\prime)\Bigr)^{1/p}.
    \end{align*}
    The co-optimal transport distance is then
    \begin{align*}
        d_{\mathrm{COOT},p}(H,H^\prime)
        =\inf_{\substack{\pi^v\in\Pi(\mu,\mu^\prime)\\ \pi^e\in\Pi(\nu,\nu^\prime)}}\mathrm{dis}_{\mathrm{COOT},p}(\pi^v,\pi^e),
    \end{align*}
    inducing a pseudo-metric on the space of measure hypernetworks\cite{chowdhury2024hypergraph}.

    \subsection{Measure Topological Network}\label{ssec:measure_topological_network}
    A \emph{measure topological network} is defined as the triple
    \begin{align*}
        P \;=\;\bigl((X,\,k,\,\mu),\;\,(Y,\,\iota,\,\nu),\;\omega\bigr),
    \end{align*}
    which integrates geometric, topological, and incidence information into a unified measure‐theoretic framework\cite{zhang2025topological}. Below we describe each component in detail.

\bigskip
    \textbf{Geometric Component \(\,(X,\,k,\,\mu)\).}  
    \begin{itemize}
        \item \(X = \{x_1,\dots,x_N\}\subset \mathbb{R}^d\) is a finite point cloud representing the raw data samples.  
        \item \(k\colon X\times X\to\mathbb{R}\) is a symmetric kernel or similarity function; for example, one may take
        \[
            k(x,x') \;=\; \exp\bigl(-\|x - x'\|^2/\sigma^2\bigr)
            \quad\text{or}\quad
            k(x,x') \;=\;\|x - x'\|,
        \]

        local geometric affinities or pairwise distances.  
        \item \(\mu\) is a probability measure supported on \(X\), often chosen to be the uniform distribution \(\mu(\{x_i\})=1/N\).  This measure allows us to speak of “mass” at each data point and to transport mass in later constructions.
    \end{itemize}

\bigskip
    \textbf{Topological Component \(\,(Y,\,\iota,\,\nu)\).}  
    \begin{itemize}
        \item \(Y\) is a locally compact Polish space whose points correspond to homology generators (e.g.\ cycles) extracted via persistent homology.  
        \item \(\iota\colon Y\to\Lambda\) is a continuous map into the persistence‐diagram domain
        \(\Lambda = \{(b,d)\in\mathbb{R}^2\mid d>b\ge0\}\).  Under \(\iota\), each generator \(y\in Y\) is sent to its birth–death pair \(\iota(y)=(b_y,d_y)\).  
        \item \(\nu\) is a Radon measure on \(Y\) such that the push‐forward \(\iota_{\#}\nu\) coincides with the usual persistence‐diagram measure on \(\Lambda\).  In practice one may take \(\nu\) to assign equal mass to each cycle representative in a given homological dimension.
    \end{itemize}

\bigskip
    \textbf{Incidence Function \(\omega\colon X\times Y\to\mathbb{R}\).}  
    The function \(\omega\) records the binary membership of points in cycles:
    \begin{align*}
        \omega(x,y) \;=\;
        \begin{cases}
            1, & x \text{ is a vertex of the cycle represented by } y,\\
            0, & \text{otherwise}.
        \end{cases}
    \end{align*}
    By treating \(\omega\) as a measurable kernel, we couple the geometric and topological parts: mass transported between two point clouds in \(X\) can be coherently matched with transport of their associated cycles in \(Y\).

    \bigskip

    The measure topological network \(P\) simultaneously captures:
    \begin{itemize}
        \item \emph{Metric structure} through \((X,k,\mu)\), enabling geometry‐aware transport;
        \item \emph{Topological features} via \((Y,\iota,\nu)\), preserving the birth–death statistics of homology classes;
        \item \emph{Higher‐order relation} through \(\omega\), enforcing consistency between points and the cycles they generate.
    \end{itemize}

    \subsection{Topological Optimal Transport (TpOT)}

    Given two measure topological networks
    \[
        P = \bigl((X,k,\mu),\,(Y,\iota,\nu),\,\omega\bigr)
        \quad\text{and}\quad
        P' = \bigl((X',k',\mu'),\,(Y',\iota',\nu'),\,\omega'\bigr),
    \]
    The \emph{Topological Optimal Transport} (TpOT) distance of order \(p\) then is defined by
    \begin{equation}\label{eq:tpot}
        d_{\mathrm{TpOT},p}(P,P')
        \;=\;
        \inf_{\substack{\pi^v \,\in\,\Pi(\mu,\mu')\\[3pt]
                      \pi^e \,\in\,\Pi_{\mathrm{adm}}(\nu,\nu')}}
        \Bigl[
          \mathcal{L}_{\mathrm{geom}}(\pi^v)
          \;+\;
          \mathcal{L}_{\mathrm{topo}}(\pi^e)
          \;+\;
          \mathcal{L}_{\mathrm{hyper}}(\pi^v,\pi^e)
        \Bigr]^{1/p},
    \end{equation}
    where the infimum is taken over all vertex–vertex couplings \(\pi^v\) and admissible topology couplings \(\pi^e\)\cite{zhang2025topological}.  Intuitively, \(\pi^v\) matches data points in \(X\) with those in \(X'\), while \(\pi^e\) matches homology generators in \(Y\) with those in \(Y'\).  The three distortion terms quantify mismatches of geometry, topology, and incidence structure, respectively.

    \bigskip
    \textbf{Geometric distortion \(\mathcal{L}_{\mathrm{geom}}\).}  
    This term generalises the Gromov–Wasserstein discrepancy to our kernelized setting:
    \begin{equation}
        \mathcal{L}_{\mathrm{geom}}(\pi^v)
        \;=\;
        \iint_{(X\times X')^2}
        \bigl|\,k(x_1,x_2)\;-\;k'(x'_1,x'_2)\bigr|^p
        \;\mathrm{d}\pi^v(x_1,x'_1)\,\mathrm{d}\pi^v(x_2,x'_2).
    \end{equation}
    By comparing pairwise affinities \(k\) versus \(k'\), this term ensures that the global metric relationships among points are preserved under the optimal coupling.

     \bigskip
    \textbf{Topological distortion \(\mathcal{L}_{\mathrm{topo}}\).}  
    The TpOT measures distance between persistence diagrams via a classical Wasserstein cost:
    \begin{equation}
        \mathcal{L}_{\mathrm{topo}}(\pi^e)
        \;=\;
        \int_{\bar{Y}\times\bar{Y}'}
        \bigl\lVert \iota(y)\;-\;\iota'(y')\bigr\rVert_{p}^{p}
        \;\mathrm{d}\pi^e(y,y'),
    \end{equation}
    where \(\bar{Y}=Y\cup\{\partial_Y\}\) and likewise for \(\bar{Y}'\), with \(\partial_Y\) the diagonal “null” feature.  This term aligns birth–death pairs, penalizing large shifts in feature lifetimes.

     \bigskip
    \textbf{Hypergraph incidence distortion \(\mathcal{L}_{\mathrm{hyper}}\).}  
    Finally, to couple points and cycles consistently, we have
    \begin{equation}
        \mathcal{L}_{\mathrm{hyper}}(\pi^v,\pi^e)
        \;=\;
        \int_{X\times X'\times Y\times Y'}
        \bigl|\omega(x,y)\;-\;\omega'(x',y')\bigr|^p
        \;\mathrm{d}\pi^v(x,x')\,\mathrm{d}\pi^e(y,y').
    \end{equation}
    Since \(\omega\) and \(\omega'\) are binary incidence functions, this term enforces that matched points participate in matched cycles, thereby preserving higher‐order topological connectivity.

    \bigskip
    In summary, TpOT simultaneously optimizes over correspondences of points and cycles, striking a balance between geometric fidelity, topological consistency, and cycle membership preservation.  The resulting distance \(d_{\mathrm{TpOT},p}\) defines a pseudo-metric on the space of measure topological networks, suitable for comparing complex data with both geometry and topology.

    \subsection{Geodesic Interpolation in TpOT Space} 
    
    An important property of the TpOT framework is that the space of measure topological networks endowed with the distance \(d_{\mathrm{TpOT},p}\) is a (non‐negatively curved) geodesic space.  In particular, given two networks \(P\) and \(P'\) and an optimal coupling, one can explicitly construct a constant‐speed geodesic between them via convex combinations of their data.  The following result summarises this construction.

    \medskip
    Let
    \[
        P = \bigl((X,\,k,\,\mu),\,(Y,\,\iota,\,\nu),\,\omega\bigr),
        \quad
        P' = \bigl((X',\,k',\,\mu'),\,(Y',\,\iota',\,\nu'),\,\omega'\bigr),
    \]
    and let \(\pi^v\in\Pi(\mu,\mu')\), \(\pi^e\in\Pi_{\mathrm{adm}}(\nu,\nu')\) be optimal couplings achieving the infimum in \eqref{eq:tpot}.  Then for each \(t\in[0,1]\), the interpolated network is defined as
    \begin{equation}\label{eq:geodesic}
        P_t \;=\;\Bigl(\;(\widetilde{X},\,k_t,\,\pi^v),\;(\widetilde{Y},\,\iota_t,\,\pi^e),\;\omega_t\Bigr),
    \end{equation}
    where:
    \begin{itemize}
        \item \(\widetilde{X} = X\times X'\) is the product of the two point clouds, endowed with the coupling measure \(\pi^v\).  
        \item \(\widetilde{Y} = (Y\times Y')\;\cup\;(Y\times\{\partial_{Y'}\})\;\cup\;(\{\partial_Y\}\times Y')\) augments the cycle space with diagonal placeholders to accommodate unmatched generators, carrying the coupling \(\pi^e\).  
        \item \textbf{Geometric kernel interpolation:}
        \[
            k_t\bigl((x_1,x'_1),(x_2,x'_2)\bigr)
            \;=\;
            (1-t)\,k(x_1,x_2)
            \;+\;
            t\,k'(x'_1,x'_2).
        \]
        At \(t=0\), this recovers the original kernel \(k\), and at \(t=1\) it recovers \(k'\), while for intermediate \(t\) it provides a linear blend of affinities.  
        \item \textbf{Topological coordinate interpolation:}
          \[
            \iota_t(y,y')
            \;=\;
            (1-t)\,\iota(y)
            \;+\;
            t\,\iota'(y').
          \]
        Here \(\iota(y)\) and \(\iota'(y')\) lie in the persistence diagram plane, and their convex combination traces a straight line segment between birth–death pairs.  
        \item \textbf{Hyperedge incidence interpolation:}
          \[
            \omega_t\bigl((x,x'),(y,y')\bigr)
            \;=\;
            (1-t)\,\omega(x,y)
            \;+\;
            t\,\omega'(x',y').
          \]
        This interpolation maintains fractional membership values, ensuring that the binary incidence structure of cycles deforms continuously along the geodesic.
    \end{itemize}

    \medskip
    Besides, we have the following properties: 
    \begin{itemize}
      \item {Metric geodesicity.}  The space \(\bigl(\mathcal{P}/\!\!\sim_w,\;d_{\mathrm{TpOT},p}\bigr)\) is a geodesic metric space.  
      \item {Convexity of geodesics for \(p=2\).}  When \(p=2\), every geodesic in this space is \emph{convex}.  
      \item {Non‐negative curvature.}  The metric space  \(\bigl(\mathcal{P}/\!\!\sim_w,\;d_{\mathrm{TpOT},p}\bigr)\) has curvature bounded below by zero. This guarantees convexity of the cost functional and stability of interpolation.  
    \end{itemize}

    \revision{
    \section{Proofs for Section \ref{ssec:entropy}}\label{app:proof}
    \begin{proof}[Proof of Property \ref{thm:maximal}]
        We prove the property for the vertex-perspective entropy $\mathrm{HE}_V(H)$; the proof for $\mathrm{HE}_E(H)$ follows symmetrically. By Definition \ref{def:HEV}, $p(v) = \frac{L(v)}{I_{\text{total}}}$ constitutes a discrete probability distribution over the finite active vertex set $V^*$, satisfying $p(v) > 0$ for all $v \in V^*$ and $\sum_{v \in V^*} p(v) = 1$. By Gibbs' inequality, given two discrete probability distributions $P=\{p(v)\}_{v\in V^*}$ and $Q=\{q(v)\}_{v\in V^*}$, then 
        \[
            -\sum_{v \in V^*} p(v) \ln p(v) \leq -\sum_{v\in V^*} p(v) \ln q(v)
        \]
        with equality if and only if $p(v)=q(v),$ for$ v\in V^*$.
        Substitute  $q(v) = \frac{1}{|V^*|}$ into the inequality, we obtain that the vertex-perspective entropy is positively strictly bounded:
        \[
            0 < -\sum_{v \in V^*} p(v) \ln p(v) \leq \ln |V^*|
        \]
        The upper bound $\ln |V^*|$ is achieved if and only if the probability distribution is uniform, i.e., $p(v) = \frac{1}{|V^*|}$ for all $v \in V^*$. Substituting the definition of $p(v)$, this equality holds if and only if
        \[
        \frac{L(v)}{I_{\text{total}}} = \frac{1}{|V^*|} \implies L(v) = \frac{I_{\text{total}}}{|V^*|} \quad \forall v \in V^*
        \]
        This implies that the degree $L(v)$ is a constant for all active vertices. By definition in graph theory, a hypergraph where all vertices have the identical degree is a \emph{regular hypergraph}. Dually, $\mathrm{HE}_E(H) = \ln |E^*|$ if and only if $S(e)$ is constant for all $e \in E^*$, which defines a \emph{uniform hypergraph}.
    \end{proof}
    }
    \revision{

    \begin{proof}[Proof of Theorem \ref{thm:abrupt_change}]    
        Following the algebraic framework of Hoffman and Singleton\cite{hoffman1960moore} for constraining graph topologies via Diophantine equations, our proof analyzes the entropy transition using a number-theoretic approach rather than continuous-limit approximations. Let $H^- = (V, E^-)$ denote the hypergraph strictly before $t_c$, with total incidence $I$. Let $H^+ = (V, E^+)$ denote the hypergraph at $t_c$ with the new edge $e_{new}$ connecting a vertex subset $V_{new}$ ($|V_{new}| = k$). The new total incidence is $I + k$. 

        For notational simplicity, let $\Sigma = \sum_{v \in V^*} L(v) \ln L(v)$. Using Definition \ref{def:HEV}, the vertex-perspective entropy before the transition is:
        \[
        \mathrm{HE}_V(H^-) = -\sum_{v \in V^*} \frac{L(v)}{I} \ln\left(\frac{L(v)}{I}\right) = \ln I - \frac{\Sigma}{I}.
        \]
        
        After the transition, the degrees update to $L(v)+1$ for $v \in V_{new}$, and remain $L(v)$ for $v \notin V_{new}$. The local variation term is defined as $\Delta \Sigma = \sum_{v \in V_{new}} \bigl[ (L(v)+1) \ln(L(v)+1) - L(v) \ln L(v) \bigr]$. The new entropy is:
        \[
        \mathrm{HE}_V(H^+) = \ln(I+k) - \frac{\Sigma + \Delta \Sigma}{I+k}.
        \]
        
        We proceed by analyzing the exact condition under which the entropy remains unchanged. Assume $\mathrm{HE}_V(H^-) = \mathrm{HE}_V(H^+)$. Equating the two expressions and rearranging yields:
        \[
        \ln(I+k) - \ln I = \frac{\Sigma + \Delta \Sigma}{I+k} - \frac{\Sigma}{I} = \frac{I \cdot \Delta \Sigma - k \Sigma}{I(I+k)}.
        \]
        
        Multiplying both sides by $I(I+k)$ to isolate the terms with integer coefficients, we obtain:
        \[
        I(I+k) \ln(I+k) - I(I+k) \ln I = I \cdot \Delta \Sigma - k \Sigma.
        \]
        
        We expand $\Sigma$ and $\Delta \Sigma$ into their explicit vertex summations by partitioning $V^*$ into affected vertices ($v \in V_{new}$) and unaffected vertices ($v \notin V_{new}$):
        \begin{align*}
        I \cdot \Delta \Sigma - k \Sigma =\;& I \sum_{v \in V_{new}} \bigl[ (L(v)+1) \ln(L(v)+1) - L(v) \ln L(v) \bigr] \notag \\
        &- k \left( \sum_{v \in V_{new}} L(v) \ln L(v) + \sum_{v \notin V_{new}} L(v) \ln L(v) \right).
        \end{align*}
        
        Grouping the terms for $v \in V_{new}$ and $v \notin V_{new}$ separately, and utilizing the logarithmic identity $x \ln y = \ln(y^x)$, the right-hand side becomes:
        \[
        \sum_{v \in V_{new}} \ln\left( \frac{(L(v)+1)^{I(L(v)+1)}}{L(v)^{(I+k)L(v)}} \right) - \sum_{v \notin V_{new}} \ln\left( L(v)^{k L(v)} \right).
        \]
        
        Applying the same logarithmic identities to the left-hand side and equating both sides as single logarithms of products , we obtain:
        \[
        \ln \left( \frac{(I+k)^{I(I+k)}}{I^{I(I+k)}} \right) = \ln \left( \frac{ \prod_{v \in V_{new}} (L(v)+1)^{I(L(v)+1)} }{ \prod_{v \in V_{new}} L(v)^{(I+k)L(v)} \cdot \prod_{v \notin V_{new}} L(v)^{k L(v)} } \right).
        \]
        
        Since the logarithmic function is strictly monotonic, we can remove the logarithms, which yields a strict integer multiplicative identity:
        \begin{equation}\label{eq:diophantine}
        (I+k)^{I(I+k)} \prod_{v \notin V_{new}} L(v)^{k L(v)} \prod_{v \in V_{new}} L(v)^{(I+k)L(v)} = I^{I(I+k)} \prod_{v \in V_{new}} (L(v)+1)^{I(L(v)+1)}.
        \end{equation}
        
        Equation \eqref{eq:diophantine} represents a highly constrained non-linear Diophantine equation\cite{shorey1986exponential,terai2012exponential}. By the Fundamental Theorem of Arithmetic, both sides must yield the exact same prime factorization. 
        
        Notice that the addition of $e_{new}$ acts as a local topological perturbation, yet it introduces a massive global multiplier shift via the terms $(I+k)^{I(I+k)}$ and $I^{I(I+k)}$. Because $I$ and $I+k$ generally possess distinct prime factors (e.g., they are coprime if $k=1$), satisfying this equality demands that the exact missing prime factors be supplied by the degree sequences $L(v)$ of the vertices. 
        
        In the combinatorial space of hypergraphs, the degree sequence cannot arbitrarily absorb such macroscopic algebraic shifts without fundamentally restructuring the entire graph. A solution requires an exact matching of prime factors between the global size $I$ and local degrees $L(v)$. Consequently, equality holds only for a trivially small, mathematically contrived class of degree sequences. For generic structural transitions, the prime factorizations strictly diverge, guaranteeing $\mathrm{HE}_V(H^-) \neq \mathrm{HE}_V(H^+)$. 
    \end{proof}
    }
    \revision{

    \begin{proof}[Proof of Theorem \ref{thm:dual_change}]
        Let $H^- = (V, E^-)$ denote the hypergraph strictly before $t_c$, with total incidence $I$. Let $H^+ = (V, E^+)$ denote the hypergraph at $t_c$, where $E^+ = E^- \cup \{e_{new}\}$ and the size of the new hyperedge is $S(e_{new}) = k$. The new total incidence is $I + k$. 

        Note that unlike vertex degrees which update locally, the sizes of the existing hyperedges remain strictly unchanged: $S(e)$ is constant for all $e \in E^-$. Let $\Sigma_E = \sum_{e \in E^-} S(e) \ln S(e)$.
        
        Using Definition \ref{def:HEE}, the hyperedge-perspective entropy before the transition is:
        \[
        \mathrm{HE}_E(H^-) = -\sum_{e \in E^-} \frac{S(e)}{I} \ln\left(\frac{S(e)}{I}\right) = \ln I - \frac{\Sigma_E}{I}.
        \]
        
        After the topological transition, the summation expands to include the new hyperedge $e_{new}$, while the global denominator updates to $I+k$:
        \[
        \mathrm{HE}_E(H^+) = -\sum_{e \in E^-} \frac{S(e)}{I+k} \ln\left(\frac{S(e)}{I+k}\right) - \frac{k}{I+k} \ln\left(\frac{k}{I+k}\right) = \ln(I+k) - \frac{\Sigma_E + k \ln k}{I+k}.
        \]
        
        We analyze the condition for entropy stagnation by assuming $\mathrm{HE}_E(H^-) = \mathrm{HE}_E(H^+)$. Equating the two forms and rearranging yields:
        \[
        \ln(I+k) - \ln I = \frac{I(k \ln k) - k \Sigma_E}{I(I+k)}.
        \]
        
        Similar to the proof of Theorem \ref{thm:abrupt_change}, we derive a Diophantine equation
        \begin{equation}\label{eq:diophantine_edge}
        (I+k)^{I(I+k)} \prod_{e \in E^-} S(e)^{k S(e)} = I^{I(I+k)} k^{I k}.
        \end{equation}
        
        This identity reveals a mathematically rigid algebraic dependency. The right-hand side of Equation \eqref{eq:diophantine_edge} is completely determined by the macroscopic variables: the initial total incidence $I$ and the perturbation size $k$. Conversely, the left-hand side relies heavily on the microscopic topological distribution of all prior existing hyperedge sizes $S(e)$. 
        
        By the Fundamental Theorem of Arithmetic, this equality holds if and only if both sides share the exact same prime factorization. This condition requires the term $\prod S(e)^{k S(e)}$ to exactly offset the difference in prime factors. Such a prime factor alignment between the prior topological state and the global perturbation is combinatorially improbable. Therefore, for any generic topological transition, the equality fails, confirming that $\mathrm{HE}_E(H^-) \neq \mathrm{HE}_E(H^+)$.
    \end{proof}
    }
    \revision{

    \begin{proof}[Proof of Theorem \ref{thm:iso}]
        Since the isomorphism preserves incidence, the degree mapping is preserved:
        \[
        L_1(v) = L_2(\phi(v)) \quad \text{for all } v \in V_1.
        \]
        \[
        S_1(e) = S_2(\psi(e)) \quad \text{for all } e \in E_1.
        \]
        Consequently, the total incidence \(I_{\text{total}}\) is identical for both hypergraphs, yielding identical probability distributions up to a permutation of indices. Since the Shannon entropy is permutation-invariant (symmetric with respect to its arguments), the summations yield identical scalar values.
    \end{proof}

    \begin{proof}[Proof of Theorem \ref{thm:upper_bound}]
        By construction, the active hyperedge set \(E^*\) corresponds exactly to the set of homological generators that have non-zero birth–death persistence in the chosen parameter range. The cardinality of this set is precisely the sum of the number of persistent generators, i.e., \(|E^*| = \sum_{k} |D_k|\). Substituting this into Property \ref{thm:maximal} directly yields
        \[
        \mathrm{HE}_E(H) \leq \ln(|E^*|) = \ln\left(\sum_{k} |D_k|\right).
        \]
    \end{proof}
    }

\newpage
\bibliographystyle{unsrt}
\bibliography{references}

\end{document}